\PassOptionsToPackage{dvipsnames}{xcolor}
\documentclass{article}

\usepackage[main, final, nonatbib]{neurips_2026}

\usepackage[utf8]{inputenc} 
\usepackage[T1]{fontenc}    
\usepackage[textsize=tiny]{todonotes}
\usepackage{hyperref}       
\usepackage{url}            
\usepackage{booktabs}       
\usepackage{amsfonts}       
\usepackage{nicefrac}       
\usepackage{microtype}      
\usepackage{amsmath}
\usepackage{amssymb}
\usepackage{mathtools}
\usepackage{amsthm}
\usepackage{xspace}
\usepackage{graphicx}
\usepackage{pifont}

\usepackage{multirow}
\usepackage{makecell}
\usepackage{enumitem}
\usepackage{wrapfig}
\usepackage{algpseudocode}
\usepackage{algorithm}
\usepackage{subcaption}
\usepackage[square,numbers]{natbib}
\usepackage[capitalize,noabbrev]{cleveref}

\definecolor{darkpurple}{HTML}{660066}
\definecolor{figred}{HTML}{B85450}
\definecolor{figgreen}{HTML}{82B366}
\definecolor{figgray}{HTML}{808080}

\newcommand{\probp}{\mathbb{P}}

\newcommand{\cmark}{\textcolor{green!60!black}{\ding{51}}}
\newcommand{\xmark}{\textcolor{red!70!black}{\ding{55}}}

\newcommand*\circledred[1]{%
\tikz[baseline=(char.base)]{
  \node[shape=circle, draw=BrickRed!60, fill=BrickRed!10, thick, inner sep=1pt] (char) {\scriptsize\textsf{#1}};
}}
\newcommand*\circledblue[1]{%
\tikz[baseline=(char.base)]{
  \node[shape=circle, draw=NavyBlue!60, fill=NavyBlue!10, thick, inner sep=1pt] (char) {\scriptsize\textsf{#1}};
}}
\newcommand{\method}{TSC\xspace}

\theoremstyle{plain}
\newtheorem{theorem}{Theorem}[section]

\theoremstyle{definition}

\newtheorem{assumption}[theorem]{Assumption}
\theoremstyle{remark}

\definecolor{darkpurple}{HTML}{660066}
\definecolor{figred}{HTML}{B85450}
\definecolor{figgreen}{HTML}{82B366}
\definecolor{figgray}{HTML}{808080}
\definecolor{darkgreen}{rgb}{0.0, 0.5, 0.0}
\definecolor{lightyellow}{HTML}{FFE699}
\definecolor{red_revision}{HTML}{FF0000}
\definecolor{darkblue}{HTML}{2E6EB3}
\definecolor{derkgreen}{HTML}{3E7D00}
\definecolor{darkyellow}{HTML}{D99542}
\definecolor{darkpurple}{HTML}{660066}

\usepackage{amsmath,amsfonts,bm}

\def\Figref#1{Figure~\ref{#1}}

\def\eqref#1{equation~\ref{#1}}

\def\1{\bm{1}}

\def\rvw{{\mathbf{w}}}

\DeclareMathAlphabet{\mathsfit}{\encodingdefault}{\sfdefault}{m}{sl}
\SetMathAlphabet{\mathsfit}{bold}{\encodingdefault}{\sfdefault}{bx}{n}

\def\gA{{\mathcal{A}}}

\def\gD{{\mathcal{D}}}

\def\gT{{\mathcal{T}}}

\def\gX{{\mathcal{X}}}
\def\gY{{\mathcal{Y}}}
\def\gZ{{\mathcal{Z}}}

\def\sR{{\mathbb{R}}}

\newcommand{\E}{\mathbb{E}}

\title{Targeted Synthetic Control Method}

\author{
  \textbf{Yuxin Wang}\thanks{Correspondence to: \texttt{Yuxin.Wang1@lmu.de}},\, \textbf{Dennis Frauen}, \textbf{Emil Javurek}, \textbf{Konstantin Hess}, \textbf{Yuchen Ma}, \\
  \textbf{Stefan Feuerriegel}\\
  LMU Munich \\
  Munich Center of Machine Learning (MCML)
}

\begin{document}

\maketitle

\begin{abstract}
    The \textit{synthetic control method} (SCM) estimates causal effects in panel data with a single-treated unit by constructing a counterfactual outcome as a weighted combination of untreated control units that matches the pre-treatment trajectory. In this paper, we introduce the \textbf{targeted synthetic control (TSC)} method, a new two-stage estimator that directly estimates the counterfactual outcome. Specifically, our TSC method (1)~yields a targeted debiasing estimator, in the sense that the targeted updating refines the initial weights to produce more stable weights; and (2)~ensures that the final counterfactual estimation is a convex combination of observed control outcomes to enable direct interpretation of the synthetic control weights. TSC is flexible and can be instantiated with arbitrary machine learning models. Methodologically, TSC starts from an initial set of synthetic-control weights via a one-dimensional targeted update through the weight-tilting submodel, which calibrates the weights to reduce bias of weights estimation arising from pre-treatment fit. Furthermore, TSC avoids key shortcomings of existing methods (e.g., the augmented SCM), which can produce unbounded counterfactual estimates. Across extensive synthetic and real-world experiments, TSC consistently improves estimation accuracy over state-of-the-art SCM baselines.
\end{abstract}

\section{Introduction}

\begin{wrapfigure}[19]{r}{0.4\textwidth}
    \centering
    \vspace{-0.3cm}
    \includegraphics[width=1\linewidth]{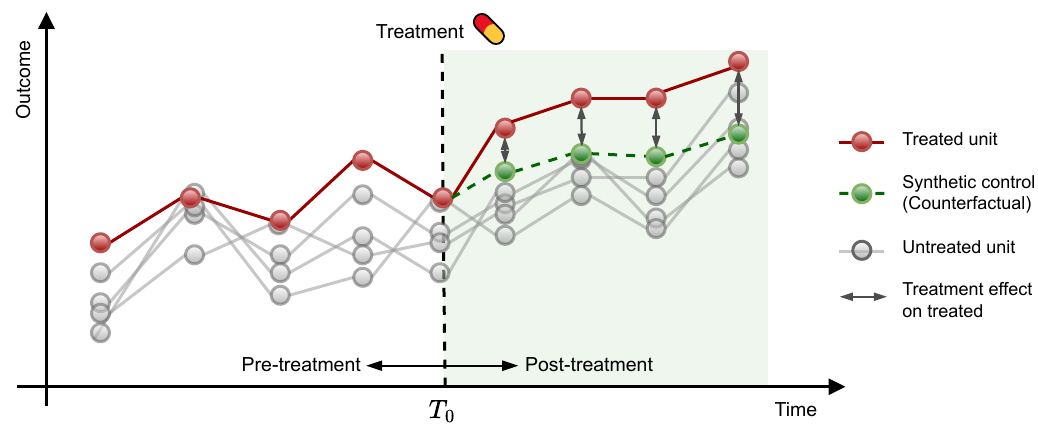}
    \caption{\textbf{Synthetic control setting}. Panel data with outcomes for single-treated unit (\textcolor{figred}{red}) and multiple control units (\textcolor{figgreen}{green}). Treatment is assigned at $T_0$, so that the pre-treatment history of the treated unit is used to learn synthetic-control weights based on other pre-treatment histories of control units (\textcolor{figgray}{gray}), while the post-treatment period is used to estimate the treatment effect between the treated unit and the synthetic control (dark arrow).}
    \label{fig:data_overview}
    \vspace{-0.4cm}
\end{wrapfigure}

The \textit{synthetic control method}~(SCM) is a widely used approach for estimating causal effects in panel data settings with a single treated unit and multiple untreated control units observed over a long pre-treatment period \citep{abadie.2003, abadie.2010}. For this, the SCM method constructs a counterfactual outcome for the treated unit as a convex combination (i.e., weighted average) of the control units based on the pre-treatment history. The estimand of interest is the causal effect on the treated unit after treatment, defined as the difference between the observed outcome of the treated unit and the outcome of the synthetic control (see Figure~\ref{fig:data_overview}).

\textbf{Example.} \emph{A canonical application of SCM is the evaluation of California’s tobacco control program (Proposition 99), enacted in 1989 \citep{abadie.2010}. Because the policy was adopted by a single state, simple before–and–after comparisons are confounded by national smoking trends. The SCM addresses this by constructing a ``synthetic'' California under no tobacco control program as a weighted combination of other U.S. states that did not adopt the program, with weights chosen to match California’s pre-1989 cigarette sales. The difference between the observed cigarette consumption after 1989 and the outcome of synthetic control then provides an estimate of the effect of the policy intervention.}

In this work, we propose the \textbf{\emph{targeted synthetic control}} (TSC) method, a new estimator for the single-treated unit setting. Our TSC method is flexible and can be instantiated with arbitrary machine learning (ML) models. \method is inspired by targeted maximum likelihood estimation (TMLE) to directly estimate the synthetic-control weights through a targeted updating procedure tailored to the single-treated unit setting. As a result, TSC has three favorable properties:
\vspace{-0.2cm}
\begin{enumerate}[leftmargin = 0.5cm]
\item[\circledblue{1}]  \emph{Interpretable synthetic-control weights.} \method ensures, by construction of the estimation procedure, that the estimated counterfactual outcome is a convex combination of observed control outcomes. The counterfactual is defined as a weighted average with nonnegative weights that sum to one, so the weights can be directly interpreted as relative contributions of control units. 
\vspace{-0.1cm}
\item[\circledblue{2}] \emph{Boundedness of estimation.} Our updating step yields bounded estimates while preserving interpretable weights: nonnegative, normalized weights keep the predicted counterfactual within the support of control outcomes and avoid extreme extrapolation.
\vspace{-0.1cm}
\item[\circledblue{3}] \emph{Estimation accuracy.} By using flexible machine learning and enforcing the residual-correcting term to zero through a one-dimensional TMLE-style targeted update, \method calibrates the synthetic-control weights using residuals of control units between outcome-regression and observed outcomes to reduce error arising from imperfect pre-treatment weight estimation. We find that this improves finite-sample stability and yields consistently improved estimation accuracy in our numerical experiments.
\vspace{-0.2cm}
\end{enumerate}

Methodologically, \method follows a two-stage estimation procedure\footnote{In the causal inference literature, procedures of this form are often referred to as \textit{meta-learners}, as they can be instantiated with arbitrary machine learning models for nuisance estimation, and then pass the estimated nuisance function into a second-stage regression, typically by regressing a constructed pseudo-outcome on covariates, to estimate the estimand of interest~\citep{kunzel.2019, robins.1994,foster.2023}.}. In the first stage, \method obtains an initial set of synthetic-control weights (e.g., obtained via the classical or regularized SCM~\citep{abadie.2003, abadie.2010}) and fits an outcome-regression model trained on the pre-treatment history for predictions of post-treatment outcomes. In the second stage, \method refines the weights through a one-dimensional \textit{``targeted''} update. Following principles from targeted maximum likelihood estimation (TMLE)~\citep{laan.2006, vanderlaan.2011}, we adopt a targeted update that leverages residual information from the control units to mitigate error arising from imperfect pre-treatment weight estimation. The resulting estimator remains a weighted combination of observed control outcomes, thereby preserving the interpretation of synthetic-control weights as relative contributions of control units.

As a result, \method offers several practical advantages over existing methods (see Table~\ref{tab:related_work}). $\bullet$ Unlike the \textbf{classical SCM}~\citep{abadie.2003, abadie.2010}, TSC updates the synthetic-control weights using residuals from a flexible machine-learning outcome-regression model, improving robustness of weight estimation. $\bullet$ Unlike the \textbf{augmented SCM}~\citep{ben-michael.2021}, \method preserves the synthetic control structure: the final counterfactual remains a convex combination of observed control-unit outcomes, with nonnegative weights that sum to one. In contrast, augmented SCM applies an additive regression correction, which can produce the counterfactual outcome that leaves the convex hull of the control outcomes and becomes unbounded. 

We make three key \textbf{contributions}\footnote{Code is available at: \url{https://github.com/Yuxin217/targeted_synthetic_control}}: \textbf{(1)} To the best of our knowledge, we are the first to link SCM to a TMLE-style targeted updating perspective while retaining interpretable synthetic-control weights. \textbf{(2)}~We propose a novel \method method that can be combined with flexible machine learning models while ensuring bounded counterfactual predictions. \textbf{(3)}~Through extensive experiments, we demonstrate state-of-the-art performance of our \method. 

\vspace{-0.2cm}
\section{Related work}
\vspace{-0.3cm}

\textbf{ML for causal inference:} There has been a rapidly growing interest in leveraging ML for causal inference from observational data, primarily in `standard' settings with many treated units and sufficient treatment overlap~\citep[e.g.,][]{vanderlaan.2011, chernozhukov.2018, shalit.2017}, where ML offers a flexible strategy to estimate nuisance components, such as outcome-regressions and propensity scores, even in high-dimensional settings. Modern causal ML estimators are often based on a two-stage framework, which combines flexible nuisance estimation with a second stage based on the efficient influence function to remove plugin bias~\citep{kennedy.2023a}. Examples include AIPTW~\citep{robins.1994}, TMLE~\citep{laan.2006, vanderlaan.2011}, the DR-learner~\citep{kennedy.2023}, and R-learner~\citep{morzywolek.2024}.
\textit{However, such methods for synthetic controls are underexplored.}

\begin{table}[htbp]
\centering
\begin{minipage}{\linewidth}
\resizebox{1\textwidth}{!}{%
\begin{tabular}{llll}
\toprule
\textbf{Aspects} & \textbf{Classical SCM}~\citep{abadie.2003, abadie.2010} & \textbf{Augmented SCM}~\citep{ben-michael.2021} & \textbf{\method} \textit{(ours)}\\
\midrule
Core idea 
& Pre-treatment balancing weights
& Outcome-model augmentation
& TMLE-style weights tilting\\

Counterfactual form 
& Convex combination of control outcomes
& Prediction $+$ additive adjustment
& Convex combination of control outcomes\\
\midrule
ML model
& \xmark~(\textit{weights only})
& \cmark~(\textit{regression})
& \cmark~(\textit{regression})\\

Interpretable weights 
& \cmark~(\textit{weighted average})
& \xmark~(\textit{mixture format})
& \cmark~(\textit{weighted average})\\

\makecell[l]{Boundedness\\(e.g., binary)} 
& \cmark~(\textit{convex hull})
& \xmark~(\textit{prediction and additive term})
& \cmark~(\textit{convex hull})\\
\bottomrule
\end{tabular}
}%
\vspace{0.1cm}
\caption{\textbf{Overview of key synthetic control estimators.}
The table compares classical SCM, augmented SCM, and \method in terms of how they construct the counterfactual and what properties are preserved by the estimator’s final form.}
\label{tab:related_work}
\end{minipage}
\vspace{-0.8cm}
\end{table}

A related literature stream extends causal ML methods for panel data, such as under difference-in-differences (DiD) designs \citep{arkhangelsky.2021, lan.2025}. We provide an extended literature review in Appendix~\ref{app:extended_related_work}. \textit{However, these approaches typically focus on settings with multiple treated units and rely on strong identifying assumptions, such as parallel trends. As a result, they are \underline{not} directly applicable to the synthetic control setting.}

\vspace{-0.1cm}
\textbf{Synthetic control method (SCM):} The literature on the SCM originates from economics and statistics literature \citep{abadie.2003, abadie.2010}\footnote{A related line of work studies prediction-based approaches in settings similar to synthetic control, often referred to as ``synthetic twins'' \citep[e.g.,][]{qian.2021, bica.2019, bellot.2021a, seedat.2022}. However, these methods focus on forecasting the treated unit’s trajectory but typically are \emph{model-specific}.}, and constructs the counterfactual of the treated unit as a convex combination of outcomes from untreated control units. Here, the counterfactual outcome is a weighted average, so the weights directly reflect the relative contribution of each control unit. Several papers have extended the SCM to different settings, such as censored data~\citep{curth.2024}, differential privacy~\citep{rho.2023}, multiple outcomes~\citep{tian.2024}, and conformal prediction~\citep{chernozhukov.2021a}. More recent work relaxes parametric assumptions by adopting more flexible identification strategies \citep{shi.2022, shi.2023}. \textit{However, in the classical SCM, the counterfactual is a weighted average based on the weights estimated by matching pre-treatment histories, with the limitation: the estimator is sensitive to finite-sample error in the estimated weights.}

\textbf{Augmented SCM.} The augmented SCM aims to reduce bias by combining synthetic-control weighting with an additional outcome model, which can be estimated flexibly via ML \citep{ben-michael.2021}. Here, we mitigate bias from imperfectly estimated pre-treatment weights using a regression-based correction constructed from the control units. While this approach can improve accuracy, it alters the structure of the estimator (see Table~\ref{tab:related_work}): \textit{the resulting counterfactual is no longer a convex combination of observed control outcomes and can become unbounded. Consequently, synthetic-control weights can no longer be interpreted in the same way as in classical SCM, where they directly quantify the contribution of each control unit to the counterfactual.}

\section{Preliminaries}
\label{sec:problem_setup}

\subsection{Setup}

We consider the classical SCM setting~\citep{abadie.2003, abadie.2010} with $i = 1, 2, \dots, N$ units observed for $t \in \gT = \{1, 2, \dots, T_0, \dots, T\}$ time periods (see \Figref{fig:data_overview}). Throughout the paper, we assume both $N$ and $T$ are fixed.
We consider a population $\left(Z, \{Y_t(0)\}_{t = 1}^{T}, \{Y_t(1)\}_{t = 1}^{T}, A \right) \sim \mathbb{P}$, where $Z \in \gZ \subseteq \mathbb{R}^p$ are time-invariant baseline covariates, $A \in \gA = \{0, 1\}$ denotes the treatment, $\{Y_t(0)\}_{t = 1}^{T}$, and $\{Y_t(1)\}_{t = 1}^{T} \in \gY \subseteq \mathbb{R}^{T}$ are potential outcome trajectories.
We further have an observed panel dataset $\{(X_i, \bar{A}_{iT}, Y_{iT})\}_{i=1}^N \sim \mathbb{P}$ with:
\vspace{-0.3cm}
\begin{itemize}[leftmargin=1.0em]
\item \textbf{Treatment:} Let $\bar{A}_{iT}$ indicate the deterministic treatment scheme $\bar{A}_{iT} = \{A_{it}\}_{t = 1}^T$, where $A_{it} = \mathbf{1}\{i=1\}\mathbf{1}\{t > T_0\}$. It represents that unit $1$ is treated at time $T_0$, where other units never receive the treatment and serve as untreated control units (sometimes also referred to as ``donors'' in the SCM literature). We assume w.l.o.g. that unit $i = 1$ is the one being treated. 
\vspace{-0.2cm}
\item \textbf{Outcomes:} We adopt the potential outcomes framework~\citep{rubin.1974}. Let $Y_{i\tilde{t}}(a_i) \in \mathbb{R}$ denote the potential outcome at post-treatment time $\tilde{t}>T_0$ and treatment assignment $A_{i\tilde{t}} = a_{i\tilde{t}}$. We have observed outcomes follow $Y_{i\tilde{t}} = A_{i\tilde{t}}Y_{i\tilde{t}}(1) +  (1-A_{i\tilde{t}})Y_{i\tilde{t}}(0)$. For the treated unit, we write $Y_{1\tilde{t}}(1) = Y_{1\tilde{t}}(0) + \tau_{\tilde{t}}$, where $\tau_{\tilde{t}}$ is the treatment effect at time $\tilde{t} > T_0$. 
\vspace{-0.2cm}
\item \textbf{Covariates:} To ease notation, for each unit $i$, let $X_i = \left(Z_i, \{Y_{it}\}_{t = 1}^{T_0}\right) $ denote concatenated vector of time-invariant baseline covariates $Z_i$ and pre-treatment outcome history $\{Y_{it}\}_{t = 1}^{T_0}$, where $Y_{it} = A_{it}Y_{it}(1) +  (1-A_{it})Y_{it}(0)$, for $t = 1, \dots, T_0$.
\end{itemize}
\vspace{-0.2cm}

\textbf{Nuisance components:} We write the \textit{synthetic-control weights} as $\rvw = \left(w_2, \dots, w_N\right) \in\Delta_{N-1}$, where $w_j \in [0,1]$ quantifies the contribution of unit $j$ to the synthetic control, with constraint $\sum_{j=2}^N w_j = 1$. We further define the \textit{outcome-regression function} $m_{\tilde{t}}(x) = \E[Y_{\tilde{t}} \mid X = x, A = 0]$. In our method, we treat both $\rvw$ and $m_{\tilde{t}}(x)$ as \textit{nuisance components}, which are quantities that must be estimated from data but are not themselves the target of inference.

\textbf{Estimation task:} Our estimand of interest is the average treatment effect on treated at the post-treatment time $\tilde{t} > T_0$, i.e.,  
$\tau_{\,\tilde{t}} = \E [Y_{\tilde{t}}(1) - Y_{\tilde{t}}(0) \mid X = X_1, A = 1]$.

In the single-treated-unit setting, we observe the treated realization $Y_{1\tilde{t}}$, while the main estimation challenge is to estimate the unobserved counterfactual $\E [Y_{1\tilde{t}}(0)\mid X = X_1, A = 1]$.

\textbf{Identifiability:} We follow the standard SCM assumptions \citep{abadie.2003, ben-michael.2021, shi.2022} to ensure identifiability of $\tau_{\tilde{t}}$:
\begin{assumption}
\label{assumptions}
For all $x\in \gX$, $a\in \gA$, it holds:
\vspace{-0.3cm}
\begin{enumerate}[leftmargin = 0.5cm]
\item[(i)]\textit{No interference:} For all units $i = 1, \dots, N$ and times $t \in \gT$, the potential outcome of unit $i$ depends only on its own treatment status, $Y_{it}(a_{1t},\dots,a_{Nt})=Y_{it}(a_{it})$.\vspace{-0.2cm}
\item[(ii)]\textit{No anticipation:} For all control units $j \in \{2, \ldots, N\}$ and all times $t \in \gT = \{1,\dots, T\}$, we have $Y_{jt} = Y_{jt}(0)$. For the treated unit $i=1$, there are no pre-treatment effects and consistency, so that $Y_{1t} = Y_{1t}(0)$ for $t \le T_0$ and $Y_{1\tilde{t}} = Y_{1\tilde{t} }(1)$ for $\tilde{t} > T_0$. \vspace{-0.25cm}
\item[(iii)]\textit{Conditional mean exchangeability:} For post-treatment time $\tilde{t}\ge T_0$, we have $\E[Y_{\tilde{t}}(0)\mid X=x, A=1]=\E[Y_{\tilde{t}}(0)\mid X=x, A=0]$.
\end{enumerate}
\end{assumption}
\vspace{-0.2cm}
Assumption~\ref{assumptions} is standard in synthetic control literature~\citep{abadie.2003, ben-michael.2021, shi.2022}: (i) imposes no interference, thus ensuring that the observed outcomes coincide with the relevant potential outcomes; (ii) rules out anticipation effects and requires control units to never be treated, so that control outcomes satisfy $Y_{jt}=Y_{jt}(0)$ for all $t$, and the treated unit satisfies $Y_{1t}=Y_{1t}(0)$ for $t\le T_0$ and $Y_{1\tilde{t}}=Y_{1\tilde{t}}(1)$ for $\tilde{t}>T_0$. Finally, (iii) assumes an invariant untreated outcome mechanism conditional on covariates $X$, which allows the conditional mean function $m_{\tilde{t}}(x)=\E[Y_{\tilde{t}}\mid X = x, A = 0]$ to be learned from control units and transported to the treated unit. 

Under Assumptions~\ref{assumptions}, our primary estimation target for $\tilde t>T_0$ is the untreated conditional mean of the treated unit,
\vspace{-0.2cm}
\begin{align}
\psi_{\tilde t}:=m_{\tilde t}(X_1).
\end{align}
Note that we observe the post-treatment outcome of the treated unit $Y_{1\tilde{t}}$, so the statistical estimand to be estimated is $\psi_{\tilde{t}}:= m_{\tilde{t}}(X_1)$, where $m_{\tilde{t}}(\cdot)$ only depends on the population of control units. The resulting treatment effect estimate is $Y_{1\tilde{t}}-\hat{\psi}{\tilde{t}}$, which need not equal the realized individual treatment effect. Thus, $m_{\tilde{t}}(\cdot)$ can be estimated from control outcomes and subsequently evaluated at $X_1$.

\subsection{Background on existing methods}

\textbf{Classical SCM:} The classical SCM constructs the counterfactual trajectory of the treated unit as an explicit convex combination of untreated control trajectories \citep{abadie.2010, abadie.2021}. The synthetic-control weights are chosen to balance pre-treatment histories and, optionally, additional covariates. Specifically, the classical SCM estimates weights $\rvw\in\Delta^{N-1}$ in the simplex 
\vspace{-0.2cm}
\begin{equation}
\label{eq:scm_simplex}
\Delta^{N-1}:=\{\rvw \in\mathbb{R}^{N-1}: w_j\ge 0,\ \sum_{j=2}^N w_j = 1\}
\end{equation}
\vspace{-0.6cm}

by solving a constrained matching problem
\vspace{-0.2cm}
\begin{equation}
\label{eq:sc_weights}
\hat{\mathbf{w}}^{\mathrm{sc}} \in  \arg\min_{\mathbf{w}\in\Delta^{N-1}}
\big\|
X_{1} - \sum_{j=2}^N w_j X_j
\big\|_{V}^{2},
\end{equation}
\vspace{-0.6cm}

where $V \succeq 0$ is a user-chosen importance matrix that prioritizes particular control units. Given the estimated weights, the synthetic control estimator (of the counterfactual) is
\vspace{-0.2cm}
\begin{equation}
\label{eq:sc_estimator}
\hat\psi_{\tilde{t}}^{\textsc{sc}} := \sum_{j=2}^{N} \hat w_{j}^{\textsc{sc}}\, Y_{j\tilde{t}}, \quad \tilde{t} > T_0,
\end{equation}
\vspace{-0.6cm}

and the corresponding treatment effect estimator is $\hat\tau_{\tilde{t}}^{\textsc{sc}} := Y_{1\tilde{t}} - \hat\psi_{\tilde{t}}^{\textsc{sc}}$.
In practice, the treated unit’s pre-treatment trajectory cannot be perfectly approximated by the synthetic-control weights, meaning that estimation error can propagate to the post-treatment period and bias the counterfactual estimate.

\textbf{Plug-in estimator:} We next introduce the so-called plug-in estimator. In the single-treated-unit setting, a direct plug-in estimate of $\psi_{\tilde{t}}=m_{\tilde{t}} (X_1)$ reduces to predicting the treated unit's untreated outcome:

\vspace{-0.4cm}
\begin{equation}
\label{eq:plugin_estimator}
\hat \psi_{\tilde{t}}^{\mathrm{plug\mbox{-}in}} =\hat m_{\tilde{t}}(X_1),
\end{equation}
\vspace{-0.5cm}

where $m_{\tilde{t}}(\cdot)$ is the control-side outcome-regression at time $\tilde{t} > T_0$. A practical advantage of Eq.~\ref{eq:plugin_estimator} is its flexibility: $\hat m_{\tilde{t}}(\cdot)$ can be learned from the control unit data using a wide range of predictive models with state-of-the-art ML models, such as representation learning~\citep{qian.2021}, recurrent neural network~\citep{bica.2019}, dynamical systems~\citep{bellot.2021a}, or controlled differential equation~\citep{seedat.2022}. However, because Eq.~\ref{eq:plugin_estimator} relies on extrapolating the learned regression to the treated unit, any estimation error in $\hat m_{\tilde{t}}(\cdot)$ translates into bias in $\hat\psi^{\mathrm{plug\mbox{-}in}}_{\tilde{t}}$, which can be non-negligible in finite samples (see discussion in~\citep{kennedy.2023a}).

\textbf{Augmented SCM:} Unlike classical SCM, augmented SCM (ASCM)~\citep{ben-michael.2021} introduces an outcome-regression adjustment to correct for imperfect pre-treatment fit. Although the original SCM weights remain constrained to the simplex, this adjustment allows the counterfactual estimate to fall outside the convex hull of the observed control outcomes. Specifically, augmented SCM incorporates an outcome-regression model, $\hat m_{\tilde{t}}(x)$, and constructs the counterfactual estimate by combining the treated unit's plug-in prediction with a weighted residual correction from the control units:
\vspace{-0.1cm}
\begin{align}
\label{eq:asc_estimator}
\hat\psi_{\tilde{t}}^{\textsc{asc}} &= \hat m_{\tilde{t}}(X_1) + \sum_{j=2}^{N} \hat w_j \bigl\{Y_{j\tilde{t}}-\hat m_{\tilde{t}}(X_j)\bigr\}.
\end{align}
\vspace{-0.5cm}

Here, $\hat w_j$ denotes the original SCM weights, which remain constrained to the simplex. The residual correction allows the resulting ASCM estimate to fall outside the convex hull of the observed control outcomes.

\vspace{-0.2cm}
\section{Targeted synthetic control method}
\begin{figure*}[htbp]
    \centering
    \includegraphics[width=1\linewidth]{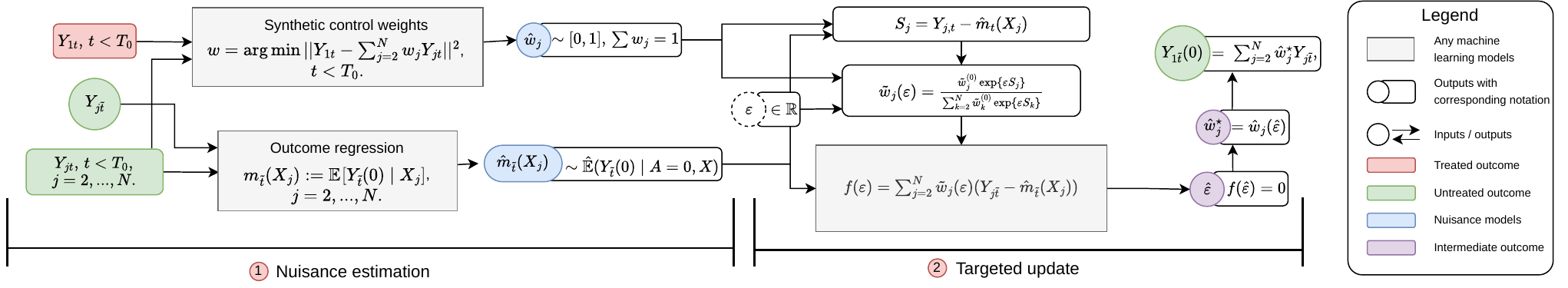}
    \caption{\textbf{Overview of \method}. Our method consists of two stages: \protect\circledred{1}~a \emph{nuisance component estimation} and \protect\circledred{2}~a \emph{targeted update}~for debiasing. In the nuisance stage, we (i) compute initial synthetic-control weights $\hat w$ by matching pre-treatment histories and (ii) fit an outcome-regression $\hat m_{\tilde{t}}(x)$ on control units using any machine-learning backbone. In the debiasing stage, we construct residual scores $S_j$ of control units and perform a one-dimensional exponential-tilting update $\tilde w(\varepsilon)$. The parameter $\hat\varepsilon$ is obtained by solving the condition $f(\hat\varepsilon)=0$. The resulting targeted weights $\hat w^\star=\tilde w(\hat\varepsilon)$ yield the final counterfactual $\hat \psi_{\tilde{t}}=\sum_{j=2}^N \hat w_j^\star Y_{j\tilde{t}}$ (and thus the treatment effect) while preserving interpretable synthetic-control weights.}
    \label{fig:framework}
    \vspace{-0.2cm}
\end{figure*}

Inspired by TMLE from classical causal inference, we now develop a two-stage estimator for the single-treated unit setting equipped with a customized TMLE-style targeting step for debiasing. Our TSC (i) debiases the initial synthetic-control weights by exponentially tilting the control-unit weights so that the weighted residuals from the outcome-regression average to zero, while keeping the weights on the simplex and therefore interpretable, and (ii) keeps the estimation of TSC in the convex hull of the outcomes of control units.

\vspace{-0.2cm}
\subsection{Overview of \method}
\label{sec:tscm}

Our \method proceeds in two stages (see Figure~\ref{fig:framework}): 
\begin{enumerate}[leftmargin = 0.5cm]
\vspace{-0.1cm}
\item[\circledred{1}] \emph{Nuisance component estimation.} We estimate the nuisance components with a flexible ML model. This includes (i)~a component for the synthetic-control weights that reweights control units to approximate the treated unit’s covariate distribution (e.g., via pre-treatment history matching) and (ii)~the outcome-regression $m_{\tilde{t}}(x)$ with cross-fitting. 
\vspace{-0.15cm}
\item[\circledred{2}] \emph{Targeted update.}  We start with the initial synthetic-control weights via a one-dimensional, TMLE-style targeted update. Specifically, we update the weights along an exponential-tilting submodel, which provides a controlled adjustment of the first-stage weights while preserving the simplex constraint (non-negativity and sum-to-one). As a result, the updated counterfactual remains a convex combination of observed control outcomes and is therefore bounded within the convex hull of control outcomes.
\vspace{-0.3cm}
\end{enumerate} 
The above procedure yields a \textit{targeted} estimator of $\psi_{\tilde{t}} = \E\!\left[Y_{\tilde{t}}\mid X = X_1, A = 1\right]$. Further, it supports interpretable weights and boundedness of the estimation. The full procedure is summarized in Algorithm~\ref{alg:tsc_tmle}. Below, we first introduce the intuition behind debiasing (§\ref{sec:intuition_of_debiasing}), which we finally turn into a targeted update using a weight-tilting submodel (§\ref{sec:targeted_update}). 

\vspace{-0.3cm}
\subsection{Intuition behind TSC}
\label{sec:intuition_of_debiasing}
\vspace{-0.1cm}

We start from the classical synthetic control approach. Specifically, we construct \emph{synthetic-control weights} $w\in\Delta^{N-1}$ by matching the controls’ pre-treatment history to that of the treated unit (e.g., pre-treatment outcomes or other chosen features), yielding $\hat \psi_{\tilde t}^{\textsc{sc}}=\sum_{j=2}^N \hat w_j\,Y_{j\tilde t}$. When accurate pre-treatment balance is infeasible, however, the resulting weights can be poorly estimated in finite samples, which makes this weights-only approximation sensitive and motivates the subsequent targeted update. 

To motivate our subsequent debiasing step, it is useful to express the corresponding estimator in a debiasing form by adding and subtracting an outcome-regression $\hat m_{\tilde t}(\cdot)$:
\vspace{-0.2cm}
{\small{\begin{align}
\hat{\psi}_{\tilde{t}}^{\textsc{sc}}
=&\;\underbrace{\sum_{j=2}^N \hat{w}_j\, \hat{m}_{\tilde{t}}(X_j)}_{\text{weighted average}}+ \underbrace{\sum_{j=2}^N \hat w_j \big\{ Y_{j{\tilde{t}}} -\hat m_{\tilde{t}}(X_j)\big\}}_{\text{residual correction}}.
\label{eq:disc_onestep_decomp}
\end{align}}}
\vspace{-0.45cm}

Here, the first term is the classical-SCM-style weighted average of outcome-regression predictions of the control units, while the second term is the residual term over the control units, which we later calibrate to zero via updating the weights.

\textbf{Why we need $\hat{m}_{\tilde{t}}(X)$:} 
A remaining limitation of Eq.~\ref{eq:sc_estimator} is that it predicts the treated outcome at $\tilde t>T_0$ solely as a weighted average of observed control outcomes. Although SCM uses pre-treatment information to estimate the weights, it does not explicitly exploit the predictive relationship between pre-treatment characteristics and post-treatment outcomes across control units. As a result, the estimator is sensitive to finite-sample error in the estimated weights. To reduce this sensitivity, we fit an outcome-regression $\hat m_{\tilde{t}}(\cdot)$ on the control units to predict $Y_{j\tilde{t}}$ from $X_j$. This regression captures the systematic component of $Y_{j\tilde{t}}$ associated with $X_j$, so the remaining residuals $Y_{j\tilde{t}}-\hat m_{\tilde{t}}(X_j)$ are closer to mean-zero noise and can be more stably aggregated by the weighted correction term.

\subsection{Derivation of the targeted update}
\label{sec:targeted_update}

\begin{wrapfigure}[20]{r}{0.52\textwidth}
\vspace{-1.7cm}
\begin{minipage}{0.52\textwidth}
\begin{algorithm}[H]
\caption{Targeted Synthetic Control Method}
\label{alg:tsc_tmle}
\begin{algorithmic}[1]
\Require Control units $\{(X_j, Y_{j\tilde{t}})\}_{j=2}^N$, treated unit $X_1$, iterations $R$, step size $\eta > 0$
\State \textbf{/* Stage 1: Nuisance estimation */}
\State $\hat{w}^0 \gets \arg\min_{w \in \Delta^{N-1}} \| X_1 - \sum_{j=2}^N w_j X_j \|_{V}^{2}$
\State $\hat{m}_{\tilde{t}}(x) \gets \mathbb{E}[Y_{\tilde{t}} \mid X = x, A = 0]$
\State \textbf{/* Stage 2: Targeted update */}
\State $S_j \gets Y_{j, t} - \hat{m}_t(X_j)$
\State Initialize $\varepsilon \gets 0$
\For{$r = 1$ \textbf{to} $R$}
    \State $\hat{w}_j(\varepsilon) \gets \dfrac{\hat{w}^0_j \exp(\varepsilon S_j)}{\sum_{k=2}^{N} \hat{w}^0_k \exp(\varepsilon S_k)}$
    \State $g(\varepsilon) \gets \sum_{j=2}^{N} \hat{w}_j(\varepsilon) \bigl( Y_{j\tilde{t}} - \hat{m}_{\tilde{t}}(X_j) \bigr)$
    \State $\varepsilon \gets \varepsilon - \eta\, g(\varepsilon)$
\EndFor
\State $\hat{w}^\star \gets \hat{w}(\hat{\varepsilon})$
\State $\hat{\psi}^{\textsc{tsc}}_{\tilde{t}} \gets \sum_{j=2}^{N} \hat{w}^\star_j Y_{j\tilde{t}}$
\State \Return $\hat{w}^\star$, $\hat{\psi}^{\textsc{tsc}}_{\tilde{t}}$
\end{algorithmic}
\end{algorithm}
\end{minipage}
\vspace{-1em}
\end{wrapfigure}
Inspired by TMLE~\citep{laan.2006, vanderlaan.2011}, we now introduce a targeting update. The key idea of targeting is to take an initial estimate (here an initial set of synthetic-control weights) and then perform a one-dimensional update guided by Eq.~\ref{eq:disc_onestep_decomp} over a submodel, so that the resulting estimator satisfies the residual correction equation with interpretable synthetic-control weights. Below, we first specify (i)~how we obtain the initial weight estimates, (ii)~the submodel to ensure the interpretability of the synthetic-control weights, and (iii)~the targeted update. 

\textbf{(i)~Initial synthetic-control weight estimates.} The starting point is a standard synthetic control estimator for the treated unit at the post-treatment time $T$:

\vspace{-0.5cm}
\begin{equation}
\hat{\psi}_{\tilde{t}} \;=\; \sum_{j=2}^{N}\hat w^{0}_j\,Y_{j\tilde{t}}, \qquad \tilde{t} > T_0,
\label{eq:sc_plugin}
\end{equation}
\vspace{-0.4cm}

where $\hat w^{0}=(\hat w^{0}_2,\ldots,\hat w^{0}_N)$ denotes an initial set of synthetic-control weights learned by Eq.~\ref{eq:scm_simplex} (e.g., obtained by pre-treatment matching under simplex constraints).

\textbf{(ii)~Weight-tilting submodel.} We define a one-dimensional exponential tilting (softmax) submodel around $\hat w^{0}$, which provides a smooth, one-dimensional way to update the initial weights $\hat w^{0}$ while preserving the simplex constraints (non-negativity and unit sum). We thus define the, for $j=2,\ldots,N$,
\vspace{-0.2cm}
{\small\begin{equation}
\hat w_j(\varepsilon)
\;=\;
\frac{\hat w^{0}_j \exp(\varepsilon S_j)}
{\sum_{k=2}^{N}\hat w^{0}_k \exp(\varepsilon S_k)},
\label{eq:weight_tilt}
\end{equation}}
guaranteeing $\hat w_j(\varepsilon)\ge 0$ and $\sum_{j=2}^N \hat w_j(\varepsilon)=1$ for all $\varepsilon\in\mathbb{R}$. Here, $S_j$ represents \emph{relative} model-implied fit: control units with $\hat m_{\tilde{t}}(X_j)$ above (below) the initial weighted average tend to receive more (less) weight, making the update focus on correcting mismatch in model-predicted outcomes. In practice, we use a centered choice based on an outcome-regression model, e.g., $S_j = Y_{j, t} - \hat{m}_t(X_j)$, so that the tilt direction is mean-zero under the initial weights, which prevents a uniform drift in the weights and makes the update depend only on relative differences across control units. This centering improves numerical stability by keeping $\varepsilon$ on a reasonable scale and reducing sensitivity to the overall level of $\hat m_{\tilde{t}}(x)$.

\textbf{(iii)~TMLE-style targeted update.}
We now describe how the targeted update is implemented in practice by solving the residual-correction equation. Given the tilted weights $\hat w(\varepsilon)$, our targeting goal is to update $\varepsilon$ so that the weighted control residuals,
\vspace{-0.2cm}
\begin{equation}
f(\varepsilon) \;=\; \sum_{j=2}^{N}\hat w_j(\varepsilon)\Big(Y_{j{\tilde{t}}}-\hat{m}_{\tilde{t}}(X_j)\Big), \, \hat{\varepsilon} = \arg \min_{\varepsilon} \log \big[\sum_{j=2}^N \hat{w}_j^0 \exp(\varepsilon (Y_{j{\tilde{t}}} - \hat{m}_{\tilde{t}}(X_j))\big],
\label{eq:tmle_score}
\end{equation}
are calibrated to zero, i.e., $\hat{\varepsilon}: f(\hat{\varepsilon}) = 0$. Since $\varepsilon$ is one-dimensional, we can compute $\hat\varepsilon$ efficiently. In practice, we obtain $\hat\varepsilon$ by solving the equivalent convex optimization problem in Eq.~\ref{eq:tmle_score}, for which the first-order derivative with respect to $\varepsilon$ is the same as Eq.~\ref{eq:tmle_score}. In other words, $\hat\varepsilon$ can be viewed as the maximum likelihood estimate of the fluctuation parameter in the weight-tilting submodel for the control weight distribution.

\textbf{Targeted estimator.} We now write down our final estimator. Let $\hat\varepsilon$ denote the root of the estimating equation $f(\varepsilon)=0$. We denote the \textit{targeted} synthetic-control weights by $\hat w^\star=\hat w(\hat\varepsilon)$. The final \method estimator is then given by

\vspace{-0.8cm}
\begin{equation}
\hat{\psi}^{\textsc{tsc}}_{\tilde{t}} =  \sum_{j=2}^{N}\hat w^\star_j\,Y_{j{\tilde{t}}}.
\label{eq:tmle_final}
\end{equation}
\noindent This estimator thus performs \textit{targeting} by updating the initial synthetic-control weights, so that the residual-correction equation is satisfied while ensuring that the estimator is a weighted combination of observed control outcomes.

Intuitively, the above construction has several advantages. (i)~It preserves the simplex constraint on the weights and therefore maintains the interpretation of synthetic-control weights as relative contributions to the counterfactual outcome. (ii)~It allows us to use flexible ML methods to estimate the nuisance components, including the outcome-regression. (iii)~It enforces the vanishing of the first-order estimation error of the outcome-regression model (see next section). 

\vspace{-0.2cm}
\subsection{Properties of \method}
\label{sec:properties}
\vspace{-0.2cm}
Below, we state two favorable properties of our \method method, namely, robustness of misspecification of the outcome-regression model and bounded outcomes.

\emph{Robustness to misspecification of outcome-regression model and imperfect fit of weight estimation:} Our approach follows a TMLE-style targeting principle, but targets the synthetic-control weights rather than the outcome regression. Specifically, we update the initial weights along an exponential-tilting submodel and select $\hat w^\star$ to (approximately) satisfy a residual-balancing condition on the controls, i.e., $\sum_{j=2}^N \hat w_j^\star\{Y_{j{\tilde{t}}}-\hat m_{\tilde{t}}(X_j)\}
= \sum_{j=2}^N \hat w_j^\star Y_{j{\tilde{t}}} - \sum_{j=2}^N \hat w_j^\star\hat m_{\tilde{t}}(X_j) \approx 0$. By calibrating the weights using control-unit residuals, this targeted update mitigates bias from imperfect pre-treatment weight estimation, and the final estimator is robust to the misspecification of $\hat m_{\tilde{t}}(\cdot)$ by enforcing Eq.~\ref{eq:tmle_score} to zero.

\begin{theorem}[Boundedness]
\label{thm:boundedness_binary}
Assume $Y_{j{\tilde{t}}}\in[a, b]$, where $a,b \in \mathbb{R}$, $a<b$, for all control units $j=2,\ldots,N$ and the targeted weights satisfy $\hat w^\star_j \ge 0$, $\sum_{j=2}^{N}\hat w^\star_j = 1$. Then our \method is bounded, $a \le \hat \psi_{{\tilde{t}}}^{\textsc{tsc}} \le b$.
\end{theorem}
\vspace{-0.2cm}
\noindent\emph{Proof.} See Appendix~\ref{app:proof_bound}. 

\textbf{Remark.} In classical TMLE, the targeting step typically is a one-dimensional update of the outcome-regression. For binary outcomes, this is commonly implemented on the logit-scale to ensure the updated predictions remain bounded. In the single-treated-unit setting, we do not directly tailor the TMLE idea to the new setting, but tilt the synthetic-control weights, so the final counterfactual is always a convex combination of observed control outcomes. Consequently, boundedness is inherent, even for binary outcomes, without requiring a logit link or any additional transformation for the outcome-regression model.

\textbf{Asymptotic properties:} Establishing asymptotic guarantees in the single-treated-unit setting remains challenging, as standard semiparametric arguments do not directly apply. In Appendix~\ref{app:population_asymptotics}, we consider a population-level extension under a joint large-$N_0$, large-$T_0$ regime and establish asymptotic normality for a corresponding orthogonal estimator under stronger assumptions on stability, transportability, and overlap, together with suitable regularity conditions. Importantly, this result concerns the population-level estimator rather than the original TSC estimator, whose asymptotic properties require further investigation.

\subsection{Differences to augmented SCM}
\label{sec:ascm_aipw}

In this section, we compare \method to the augmented SCM to clarify the limitations of the latter. The augmented SCM estimator follows a ``prediction $+$ residual correction” decomposition:
\vspace{-0.2cm}
\begin{align}
\hat{\psi}_{{\tilde{t}}}^{\textsc{asc}} = \underbrace{\hat m_{\tilde{t}}(X_1)}_{\text{plugin term}} + \underbrace{\sum_{j=2}^N \hat w_j\bigl\{Y_{j{\tilde{t}}}-\hat m_{\tilde{t}}(X_j)\bigr\}}_{\text{residual correction term}}.
\end{align}
\vspace{-0.4cm}

The augmented SCM in~Eq.~(\ref{eq:asc_estimator}) can be interpreted as a plug-in prediction $\hat m_{\tilde{t}}(X_1)$ plus a regression correction. The final outcome does \textit{not} have the boundedness guarantee, since, even if each term individually lies in $[a,b]$, the outcome model prediction in Eq.~(\ref{eq:asc_estimator}) may \textit{not} be in $[a,b]$. 

\textbf{Why our targeted update is beneficial over augmented SCM:} A simple additive residual correction can work well in practice, but it debiases by adding an augmentation term to the plug-in estimator, rather than by updating the synthetic-control weights. Hence, the targeted update has two important advantages: \textbf{(1)}~it preserves the weight-based structure of synthetic control estimators, so the contributions of control units are interpretable through weights, and \textbf{(2)}~it enforces the final counterfactual estimate remains a convex combination of observed synthetic control estimators, so outcomes remain within their natural bounds.

\vspace{-0.2cm}
\section{Experiments}
\vspace{-0.2cm}

We follow best practice in causal ML (e.g., \citet{laan.2006, ben-michael.2021, shi.2022}) to perform experiments that demonstrate the effectiveness of \method across different types of outcomes using synthetic data. Synthetic data has the advantage of providing access to the ground-truth treatment effect and thus allows for direct comparison against oracle estimates. We instantiate all models with the same neural network architectures and hyperparameters. Implementation details and additional experimental results are provided in Appendix~\ref{app:implementation_details} and Appendix~\ref{app:add_exp_results}, respectively.
\vspace{-0.2cm}
\subsection{Synthetic data}
\vspace{-0.4cm}

\begin{table*}[htbp]
    \centering
    \small
\begin{minipage}[t]{0.49\linewidth}
\centering
\begingroup
\renewcommand{\arraystretch}{1.1}
\resizebox{\linewidth}{!}{%
\begin{tabular}{clcccc}
\toprule
\makecell[c]{Prediction\\horizon} & Methods & Linear & Hinge & Quadratic & Time-varying \\
\midrule
\multirow{4}{*}{1} & Classical SCM & $0.922 \pm 0.146$ & $0.988 \pm 0.163$ & $1.030 \pm 0.186$ & $0.888 \pm 0.126$ \\
 & Plug-in estimator & $1.053 \pm 0.245$ & $1.118 \pm 0.292$ & $1.434 \pm 0.401$ & $0.987 \pm 0.212$ \\
 & Augmented SCM & $1.111 \pm 0.245$ & $1.186 \pm 0.292$ & $1.458 \pm 0.401$ & $1.049 \pm 0.212$ \\
 & \method (\textit{ours}) & $\mathbf{0.892 \pm 0.134}$ & $\mathbf{0.924 \pm 0.139}$ & $\mathbf{1.022 \pm 0.188}$ & $\mathbf{0.879 \pm 0.126}$ \\
\midrule
\multirow{4}{*}{5} & Classical SCM & $0.975 \pm 0.233$ & $0.964 \pm 0.230$ & $1.044 \pm 0.226$ & $0.935 \pm 0.236$ \\
 & Plug-in estimator & $0.993 \pm 0.346$ & $1.001 \pm 0.364$ & $1.152 \pm 0.422$ & $1.012 \pm 0.365$ \\
 & Augmented SCM & $1.039 \pm 0.348$ & $1.051 \pm 0.367$ & $1.211 \pm 0.425$ & $1.046 \pm 0.376$ \\
 & \method (\textit{ours}) & $\mathbf{0.970 \pm 0.231}$ & $\mathbf{0.937 \pm 0.235}$ & $\mathbf{1.035 \pm 0.230}$ & $\mathbf{0.934 \pm 0.233}$ \\
\midrule
\multirow{4}{*}{10} & Classical SCM & $0.810 \pm 0.203$ & $0.771 \pm 0.210$ & $0.866 \pm 0.228$ & $0.871 \pm 0.231$ \\
 & Plug-in estimator & $1.107 \pm 0.185$ & $1.113 \pm 0.180$ & $1.032 \pm 0.156$ & $1.127 \pm 0.191$ \\
 & Augmented SCM & $1.113 \pm 0.191$ & $1.115 \pm 0.182$ & $1.080 \pm 0.164$ & $1.141 \pm 0.197$ \\
 & \method (\textit{ours}) & $\mathbf{0.778 \pm 0.175}$ & $\mathbf{0.727 \pm 0.178}$ & $\mathbf{0.806 \pm 0.196}$ & $\mathbf{0.849 \pm 0.207}$ \\
\bottomrule
\end{tabular}%
}
\endgroup
\par\smallskip
{\small (a) Continuous outcomes\par}
\end{minipage}%
\hfill
\begin{minipage}[t]{0.49\linewidth}
\centering
\begingroup
\renewcommand{\arraystretch}{1.1}
\resizebox{\linewidth}{!}{%
\begin{tabular}{clcccc}
\toprule
\makecell[c]{Prediction\\horizon} & Methods & Linear & Hinge & Quadratic & Time-varying \\
\midrule
\multirow{4}{*}{1} & Classical SCM & $0.706 \pm 0.119$ & $0.840 \pm 0.063$ & $0.618 \pm 0.106$ & $0.703 \pm 0.082$ \\
 & Plug-in estimator & $0.776 \pm 0.153$ & $1.158 \pm 0.392$ & $1.075 \pm 0.309$ & $1.638 \pm 0.453$ \\
 & Augmented SCM & $0.850 \pm 0.148$ & $1.194 \pm 0.349$ & $1.112 \pm 0.316$ & $1.650 \pm 0.429$ \\
 & \method (\textit{ours}) & $\mathbf{0.423 \pm 0.039}$ & $\mathbf{0.524 \pm 0.074}$ & $\mathbf{0.386 \pm 0.067}$ & $\mathbf{0.486 \pm 0.079}$ \\
\midrule
\multirow{4}{*}{5} & Classical SCM & $0.588 \pm 0.073$ & $0.653 \pm 0.044$ & $0.606 \pm 0.071$ & $0.575 \pm 0.090$ \\
 & Plug-in estimator & $0.592 \pm 0.145$ & $1.103 \pm 0.271$ & $0.725 \pm 0.132$ & $0.874 \pm 0.270$ \\
 & Augmented SCM & $0.506 \pm 0.183$ & $1.096 \pm 0.281$ & $0.738 \pm 0.160$ & $0.925 \pm 0.317$ \\
 & \method (\textit{ours}) & $\mathbf{0.316 \pm 0.041}$ & $\mathbf{0.478 \pm 0.081}$ & $\mathbf{0.451 \pm 0.061}$ & $\mathbf{0.312 \pm 0.063}$ \\
\midrule
\multirow{4}{*}{10} & Classical SCM & $0.593 \pm 0.090$ & $0.452 \pm 0.086$ & $0.694 \pm 0.103$ & $0.462 \pm 0.079$ \\
 & Plug-in estimator & $0.558 \pm 0.096$ & $0.920 \pm 0.165$ & $0.766 \pm 0.109$ & $0.767 \pm 0.180$ \\
 & Augmented SCM & $0.678 \pm 0.144$ & $0.902 \pm 0.164$ & $0.815 \pm 0.094$ & $0.781 \pm 0.185$ \\
 & \method (\textit{ours}) & $\mathbf{0.420 \pm 0.088}$ & $\mathbf{0.448 \pm 0.051}$ & $\mathbf{0.562 \pm 0.073}$ & $\mathbf{0.408 \pm 0.062}$ \\
\bottomrule
\end{tabular}%
}
\endgroup
\par\smallskip
{\small (b) Binary outcomes\par}
\end{minipage}%

    \caption{\textbf{Results for experiments with synthetic data.} Mean $\pm$ standard error across five runs, reported for prediction horizons of 1, 5, 10 steps ahead. $\Rightarrow$ \textit{Our \method performs best across all settings.}}
    \label{tab:syn_exp_results}
\end{table*}

\textbf{Dataset:} We simulate samples from multiple data-generating functions, including (1)~\textit{linear}, (2)~\textit{hinge}, (3)~\textit{quadratic}, and (4)~\textit{time-varying} outcome functions. Details are in Appendix~\ref{app:data_generation}. For each function, we compare the root mean squared error and report the mean $\pm$ standard error across five runs for prediction horizons of 1, 5, and 10 steps ahead. Altogether, we generate over 50 different datasets under varying scenarios. 

\textbf{Results:} We compare our \method against the \textbf{classical SCM}~\citep{abadie.2010}, a \textbf{plug-in} estimator (Eq.~\ref{eq:plugin_estimator}), and the \textbf{augmented SCM}~\citep{ben-michael.2021}. The results are in Table~\ref{tab:syn_exp_results}~(a) (for continuous outcomes) and Table~\ref{tab:syn_exp_results}~(b) (for binary outcomes). \textit{Across all datasets, \method consistently achieves the best performance.} The improvement for binary outcomes (Table~\ref{tab:syn_exp_results}~(b)) is especially noteworthy. Here, our \method reduces the RMSE relative to the augmented SCM by $\sim$ 18\%. We attribute this improvement to the targeted update, which enforces bounded predictions and avoids unstable extrapolation.

\textbf{Interpretability:} We further demonstrate the interpretability of the weights produced by \method. Figure~\ref{fig:weight} compares the initial synthetic control (light colors) with the final weights (dark colors). \textit{Importantly, we confirm that the targeted weights remain nonnegative and sum to one, allowing them to be directly interpreted as relative contributions of control units to the constructed counterfactual.}

\vspace{-0.1cm}
\begin{wrapfigure}[12]{l}{0.4\textwidth}
    \vspace{-0.5cm}
    \centering
    \includegraphics[width=1\linewidth]{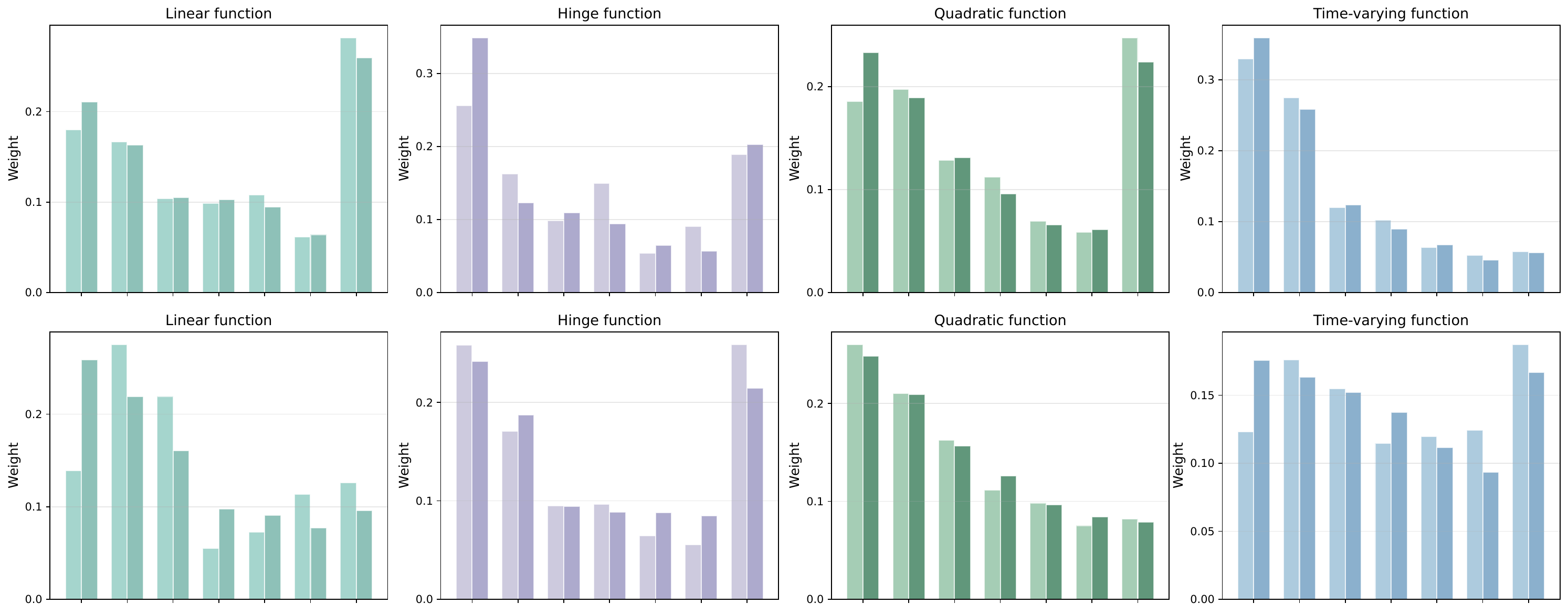}
    \vspace{-0.4cm}
    \caption{\textbf{Insights of weights.} Bars show the weights of control units used to construct the counterfactual (\emph{top row}: continuous; \emph{bottom row}: binary). Shown are the initial synthetic control weights (light colors) and the final weights (dark colors).}
    \label{fig:weight}
    \vspace{-0.2cm}
\end{wrapfigure}

\textbf{Boundedness:} Table~\ref{tab:binary_bound} empirically validates the boundedness property of our method (Theorem~\ref{thm:boundedness_binary}) and thus highlights a key advantage of \method over the augmented SCM: the targeted update guarantees that counterfactual predictions remain bounded. This boundedness ensures that the original interpretation of synthetic-control weights as defined in the classical SCM. In contrast, the augmented SCM produces many out-of-bounds estimates.

\begin{wraptable}[8]{r}{0.4\textwidth}
\vspace{-0.5cm}
\centering
\resizebox{0.4\textwidth}{!}{%
\begin{tabular}{lccccc}
\toprule
Bound & Method & Linear & Hinge & Quadratic & Time-varying \\
\midrule
\multirow{2}{*}{\centering Upper bound} & Augmented SCM & \textcolor{red}{3.33\%} & \textcolor{red}{13.33\%} & \textcolor{red}{20.00\%} & \textcolor{red}{26.67\%} \\
 & TSC \textit{(ours)} & \textcolor{ForestGreen}{0.00\%} & \textcolor{ForestGreen}{0.00\%} & \textcolor{ForestGreen}{0.00\%} & \textcolor{ForestGreen}{0.00\%} \\
\midrule
\multirow{2}{*}{\centering Lower bound} & Augmented SCM & \textcolor{red}{13.33\%} & \textcolor{red}{40.00\%} & \textcolor{red}{13.33\%} & \textcolor{red}{13.33\%} \\
 & TSC \textit{(ours)} & \textcolor{ForestGreen}{0.00\%} & \textcolor{ForestGreen}{0.00\%} & \textcolor{ForestGreen}{0.00\%} & \textcolor{ForestGreen}{0.00\%} \\
\bottomrule
\end{tabular}
}
\parbox{0.52\textwidth}{
  \raggedright
  {\scriptsize{\textsuperscript{*} Violation rates above zero are shown in \textcolor{red}{red}.}}
}
\vspace{-0.3cm}
\caption{Percentage of counterfactual estimations outside $[0,1]$ on binary outcomes. $\Rightarrow$ \textit{\method achieves \textcolor{ForestGreen}{0.00\%} violation rate; augmented SCM does not.}}
\label{tab:binary_bound}
\vspace{-1em}
\end{wraptable}

\vspace{-0.2cm}
\subsection{Real-world data}
\vspace{-0.2cm}

\textbf{Data:} We conduct case studies using two real-world datasets. \textbf{(a) Turnout data:} We use the Turnout rate dataset~\citep{springer.2014}, which records U.S. state-level elections and voter turnout rates from 1920 to 2000. As the targeted unit, we focus on New Hampshire, which adopted same-day voter registration in 1996, and use other states as controls. \textbf{(b) California tobacco data:} We use the California smoking dataset from \citet{abadie.2010}, which records annual per-capita cigarette sales across U.S. states from 1970 to 2000. The targeted unit is California, which enacted Proposition 99 (i.e., a larger health reform, which also included a tax increase on tobacco) in 1988, and we use 38 other states as controls.

\vspace{-0.1cm}
\begin{figure}[htbp]
    \centering
    \begin{subfigure}[b]{0.48\linewidth}
        \includegraphics[width=\linewidth]{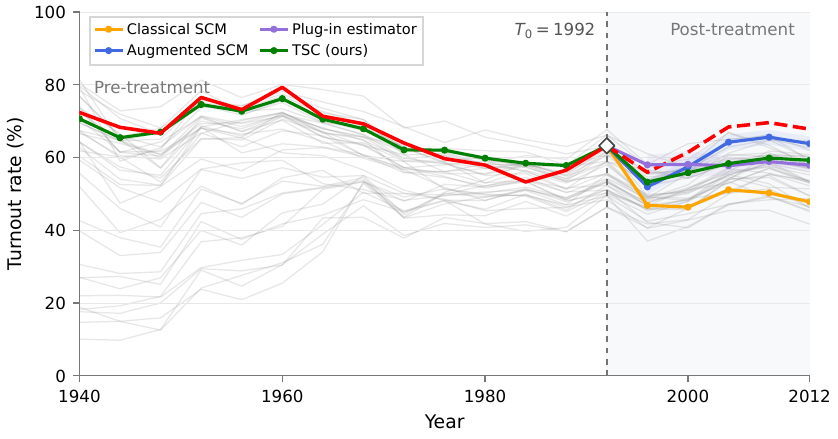}
        \caption{Turnout data}
        \label{fig:turnout_data}
    \end{subfigure}
    \hfill
    \begin{subfigure}[b]{0.48\linewidth}
        \includegraphics[width=\linewidth]{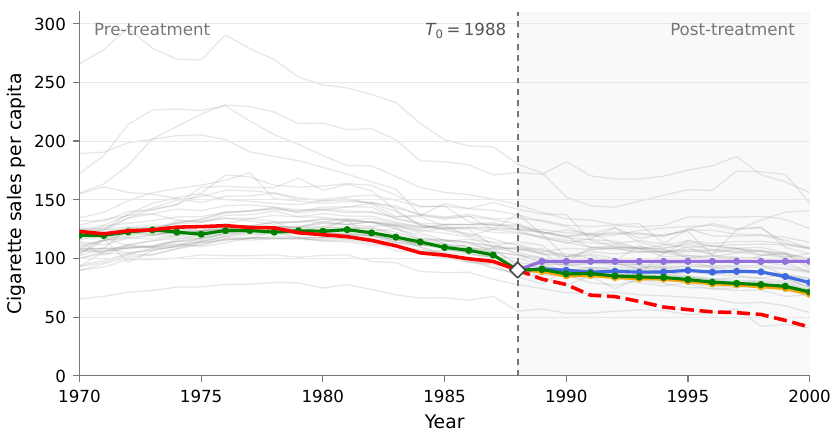}
        \caption{California tobacco data}
                \label{fig:smoke_data}
    \end{subfigure}
    \vspace{-0.2cm}
    \caption{\textbf{Real-world experiments.} The trajectories of the control units are shown in \textcolor{gray}{gray}, and the treated unit in \textcolor{red}{red}. We plot counterfactual estimates from augmented SCM (\textcolor{blue}{blue}), classical SCM (\textcolor{orange}{orange}), the plug-in estimator (\textcolor{darkpurple}{purple}), and \method (\textcolor{ForestGreen}{green}). The vertical dashed line indicates the treatment time $T_0$. \method matches the pre-treatment history closely in both datasets.}
    \label{fig:real_world_data}
    \vspace{-0.2cm}
\end{figure}

\textbf{Results:} In Figure~\ref{fig:real_world_data}, we plot the predicted counterfactual trajectory of the treated unit under no treatment for both datasets. \textbf{(a) Turnout data:} Same-day registration is expected to increase voter turnout, reflected in the observed increase in the treated unit after treatment (shown in \textcolor{red}{red}). Our \method (\textcolor{ForestGreen}{green}) yields a stable counterfactual trajectory, implying a large positive treatment effect. In contrast, the classical SCM exhibits a systematic downward deviation, while the augmented SCM shows an upward deviation, indicating bias in the constructed controls. \textbf{(b) California tobacco data:} Proposition 99 is expected to reduce cigarette consumption. Our \method (\textcolor{ForestGreen}{green}) produces a stable counterfactual that remains within the range of control outcomes, suggesting a substantial negative treatment effect after 1988 in line with the finding from the classical SCM in \citet{abadie.2010}. The plug-in and augmented SCM both show less plausible results in the post-treatment period, whereas \method tracks the pre-treatment trajectory closely and yields bounded predictions throughout. We provide further discussion of the theoretical properties and limitations of TSC in Appendix~\ref{app:discussion}.

\section*{Acknowledgement}
Our research is supported by the DAAD programme Konrad Zuse Schools of Excellence in Artificial Intelligence, sponsored by the Federal Ministry of Research, Technology and Space. This work has also been supported by the German Federal Ministry of Education and Research (Grant: 01IS24082).

\bibliographystyle{abbrvnat}
\bibliography{literature}


\newpage
\appendix
\onecolumn

\section{Extended related work}
\label{app:extended_related_work}

\textbf{Semiparametric inference and orthogonal learning:} The concept of Neyman orthogonality is deeply rooted in semiparametric efficiency theory~\citep{kennedy.2023a, vaart.2000}. Neyman-orthogonal and efficient influence function-based estimators have a long tradition in causal inference, primarily for the estimation of average treatment effects. Examples include the AIPTW estimator~\citep{robins.1994}, TMLE~\citep{laan.2006,laan.2021}, and the Double~ML framework~\citep{chernozhukov.2018}. Recently, the concept of Neyman orthogonality has been extended to HTEs~\citep{foster.2023}, which allowed the construction of various orthogonal learners, including the DR- and R-learner for conditional average treatment effects~\citep{kennedy.2023, morzywolek.2024, nie.2020}. \textit{None of these works are tailored to the single-treated unit setting.}

\textbf{Heterogeneous treatment effects estimation over time:} There has been much research on the estimation of heterogeneous treatment effects (HTE) over time, often leveraging several backbone architectures. Early approaches build on recurrent and state-space sequence models that learn latent disease dynamics and simulate counterfactual treatment trajectories \citep{lim.2018, li.2021}. More recent methods leverage representation learning and deep generative modeling, including variational and adversarial formulations, to capture complex, high-dimensional temporal confounding and treatment responses \citep{bica.2019, melnychuk.2022, seedat.2022}. In parallel, transformer-style and other attention-based backbones have been proposed to improve long-horizon modeling and to better exploit irregular, multivariate patient histories, but with various strategies for confounder adjustment~\citep{hess.2024, hess.2025a, hess.2025c, frauen.2025, wang.2025a, ma.2025c}.

\textbf{CATT estimation under other settings:} A growing literature develops more flexible, non-linear approaches for estimating CATT~\citep[e.g.,][]{nazaret.2024, tian.2023, spiess.2023}, with broad applications in public policy \citep[e.g.,][]{markle-huss.2018}. Existing methods can be broadly grouped into three strands: (a) Synthetic twin approaches~\citep[e.g.,][]{qian.2021, bica.2019, bellot.2021a} predict the treated unit from trajectories and the control pool, effectively turning the problem into a flexible forecasting task. However, these models provide no formal identifiability guarantee for the prediction. (b) Difference-in-differences (DiD) style methods~\citep[e.g.,][]{card.1993, lan.2025, arkhangelsky.2021, sun.2025} leverage non-linear learners but still rely on strong identifying restrictions, most notably a parallel-trends assumption. (c) Proxy causal inference approaches  \citep[e.g.,][]{shi.2023, liu.2024, park.2025, oriordan.2024, zeitler.2023} attempt to recover causal effects by introducing proxy variables, and identifiability hinges on additional proxy assumptions. \textit{Overall, while these non-linear methods can improve predictive fit, their causal interpretation typically requires additional, method-specific assumptions.}

\textbf{More synthetic control variants:}
A related line of work localizes control construction by combining synthetic control with matching-style restrictions. For instance, caliper/radius-based procedures form \emph{local} control pools via distances and adaptive calipers and then fit synthetic controls within these neighborhoods~\citep{che.2024}. Complementarily, matching–synthetic-control hybrids use cross-validated model averaging to balance the extrapolation bias of matching against interpolation bias of synthetic controls~\citep{kellogg.2021}. These methods can improve finite-sample stability by restricting or reweighting the control set, but they remain primarily weight-based constructions and are typically not derived from Neyman-orthogonal estimating equations. In parallel, recent work develops asymptotically unbiased or debiased synthetic-control estimators \citep{fry.2024b, fry.2025, chernozhukov.2025}, largely focusing on inference for averages of post-treatment effects rather than period-specific targets.

\newpage
\section{One-step debiasing vs. TMLE}
\label{app:eif_tmle_background}
We consider a simple example of estimating the ATE in a randomized controlled trial setting to illustrate the difference between the one-step debiasing estimator (AIPW) and TMLE. Let $Z = (X, A, Y)\sim P$ denote observed data, where $X \in \gX \subseteq \sR^{q}$ are observed covariates, $A \in \gA = \{0, 1\}$ is binary treatment, $Y \in \gY \subseteq \sR$ represents outcomes. We assume that we have an observed dataset $\gD = \{(x_i, a_i, y_i)\}_{i=1}^n$ of size $n\in N$ are sampled i.i.d from $\probp$. We use the potential outcomes framework~\citep{rubin.1974} and denote $Y(a)$ as the potential outcome corresponding to a treatment $A = a$. We define the \emph{response functions} as $\mu_a(x) = \E[Y\mid X = x, A = a]$ for $a\in \{0, 1\}$ and the \emph{propensity score} as $\pi(x) = \probp[A = 1 \mid X= x]$. We refer to these functions as \emph{nuisance functions}, denoted by $\eta = (\mu_1, \mu_0, \pi)$. Besides, we impose standard causal inference assumptions~\citep{laan.2006, chernozhukov.2018}, including consistency, overlapping assumption, and ignorability, to ensure the identification of the ATE.

\textbf{Efficient influence functions (EIFs):} Our targeted causal estimand is the ATE, and according to standard causal inference assumptions, we convert the causal estimand to a statistical estimand:
\begin{align}
\underbrace{\E[Y(1) - Y(0)]}_{\text{causal estimand}} =& \E\left[ \E[Y(1) - Y(0) \mid X = x]\right]
= \underbrace{\E[\mu_1(x) - \mu_0(x)]}_{\text{statistical estimand}}.
\end{align}

Since $\psi = \E[\mu_1(x) - \mu_0(x)]$ has a so-called von Mises or distributional Taylor expansion~\citep{vaart.2000}, we can derive the efficient influence function 
\begin{align}
\label{eq:eif_ate}
\phi = \mu_1(X) - \mu_0(X) + \frac{A}{\pi(X)} (Y - \mu_1(X)) - \frac{1-A}{1-\pi(X)} (Y - \mu_0(X)) - \psi.
\end{align}

\textbf{One-step debiased estimators (AIPW).}
Given initial nuisance estimates $\hat\eta$, a generic EIF-based \emph{one-step} debiasing estimator takes the form
\begin{align}
\label{eq:one_step_general}
\hat\psi^{\text{one-step}} 
=&\psi(\hat P)+\frac{1}{n}\sum_{i=1}^n \phi(Z_i;\hat\eta),
\end{align}
where $\psi(\hat \probp)$ is a plug-in estimator based on an estimated distribution $\hat \probp$ (equivalently, based on fitted nuisance functions), $\psi(\hat P) = \frac{1}{n}\sum_{i=1}^n[\hat\mu_1(x_i)-\hat\mu_0(x_i)]$. Plugging Eq.~(\ref{eq:eif_ate}) into Eq.~(\ref{eq:one_step_general}) yields the one-step debiasing, i.e., the augmented inverse probability weighting (AIPW) estimator
\begin{equation}
\label{eq:aipw}
\hat\psi^{\textsc{aipw}}
= \frac{1}{n}\sum_{i=1}^n
\left[ \hat\mu_1(x_i)-\hat\mu_0(x_i) +
\frac{a_i}{\hat\pi(x_i)}\bigl(y_i-\hat\mu_1(x_i)\bigr) -
\frac{1-a_i}{1-\hat\pi(x_i)}\bigl(y_i-\hat\mu_0(x_i)\bigr)
\right].
\end{equation}
Under suitable regularity conditions (often paired with cross-fitting), the AIPW estimator is asymptotically normal and achieves the semiparametric efficiency bound when the nuisance estimators are sufficiently accurate. Moreover, it enjoys the familiar \emph{double robustness} property: it remains consistent if either the propensity model $\pi$ or the outcome models $\{\mu_a\}_{a\in\{0,1\}}$ are consistently estimated (with appropriate positivity).

\textbf{Targeted maximum likelihood estimation (TMLE).}
TMLE is another EIF-based debiasing way by \emph{targeting} a nuisance function through a low-dimensional update. Starting from initial nuisance function estimates $\hat\eta$, TMLE defines a parametric fluctuation (submodel) $\{\hat\eta_\epsilon:\epsilon\in\mathbb{R}^d\}$
through $\hat\eta$ and chooses $\hat\epsilon$ by minimizing a loss (e.g., negative log-likelihood) in a way that enforces the EIF estimating equation:
\begin{align}
\label{eq:tmle_score_equation}
\frac{1}{n}\sum_{i=1}^n \phi(Z_i;\hat\eta_{\hat\epsilon}) \approx 0.
\end{align}

We take the response functions as an example. Start with 
\begin{align}
    \mu_A^\varepsilon (X) = \hat{\mu}_A(X) + \varepsilon \hat{r}(A, X),
\end{align}
where $\hat{r}(A, X) = \frac{A}{\hat\pi(X)} - \frac{1-A}{1-\hat\pi(X)}$ is the clever covariate, $\varepsilon \in \sR$.
The fluctuation parameter $\hat\varepsilon$ is obtained by maximizing the likelihood along the chosen fluctuation submodel. Equivalently, one can estimate $\hat\epsilon$ by fitting a regression of $Y$ on the clever covariate $\hat r(A,X)$ using $\hat\mu_A(X)$ as an offset and imposing no intercept term. When the outcome $Y$ is bounded (e.g., binary), it is common to use a link function such as the logit link in the targeting step so that the updated regression $\hat\mu^{\hat\epsilon}(X)$ remains within the admissible range. Then, the targeted response functions are $ \hat\mu_a^\star(X) := \hat\mu_a^{\hat\epsilon}(X)$.

Finally, the TMLE estimator is 
\begin{align}
\hat\psi^{\textsc{tmle}}
= \frac{1}{n}\sum_{i=1}^n
 \hat\mu_1^\star(x_i)-\hat\mu_0^\star(x_i).
\end{align}

\textbf{Why TMLE can be preferable in practice.}
Both the one-step estimator and TMLE share the same EIF, which means both of them can attain semiparametric efficiency under comparable nuisance function fitting conditions (e.g., cross-fitting). However, TMLE often has two practical advantages. \textit{First}, as a likelihood-based substitution estimator, TMLE inherits the natural parameter constraints of the chosen model: by updating the nuisance fit \textit{within} a \textit{constrained} likelihood submodel (e.g., via a logit link), the resulting predictions remain in the admissible range, so bounded outcomes yield bounded estimated means or probabilities. \textit{Second}, TMLE chooses the fluctuation parameter to (approximately) satisfy the EIF estimating equation in Eq.~(\ref{eq:tmle_score_equation}), absorbing the first-order bias correction into a targeted nuisance update rather than a single additive adjustment; this typically leads to improved numerical behavior and more stable finite-sample performance.
 Thus, TMLE is frequently preferred when nuisance models are flexible and finite-sample behavior matters.

\newpage
\section{Background on the augmented SCM}

The augmented SCM reduces bias by combining a synthetic-control-style weighting with an outcome-regression model fitted on control data \citep{ben-michael.2021}. Let $\hat m_{\tilde{t}}(\cdot)$ denote a regression estimator trained on the control units to predict $Y_{i{\tilde{t}}}$ from pre-treatment history, and define residualized outcomes
$\tilde Y_{i{\tilde{t}}} = Y_{i{\tilde{t}}} - \hat m_{\tilde{t}}(X_i).$
The augmented SCM estimates weights by matching residualized pre-treatment trajectories, i.e.,
\begin{equation}
\label{eq:aSC method_weights}
\hat{\mathbf{w}}^{\textsc{asc}} \in 
\arg\min_{\mathbf{w}\in\Delta^{N-1}}
\left\|
X_{1} - \sum_{j=2}^N w_j X_{j}
\right\|_{V}^{2}
 + \lambda\mathcal{R}(\mathbf{w}), \quad \forall t \leq T_0,
\end{equation}
where $\mathcal{R}(\mathbf{w})$ is an optional regularizer (e.g., ridge-type), $V \succeq 0$ is a user-chosen importance matrix that prioritizes particular control units, and $\lambda \geq 0$ controls shrinkage. The counterfactual in the augmented SCM takes the regression prediction for the treated unit and adds a weighted correction based on control residuals:
\begin{align}
\label{eq:asc_method_counterfactual}
\hat \psi_{\tilde{t}}^{\textsc{asc}}
&= \hat m_{\tilde{t}}(X_1) + \sum_{j=2}^{N} \hat w_{j}^{\textsc{asc}} 
\bigl( Y_{j{\tilde{t}}} - \hat m_{\tilde{t}}(X_j) \bigr),
\end{align}
Intuitively, $\hat{m}_{\tilde{t}}(X_1)$ provides a flexible plug-in prediction, while the weighted residual term serves as a residual correction, leveraging observed control outcomes to offset model bias. The augmented SCM can substantially improve performance when pre-treatment fit is imperfect. However, its accuracy can depend on nuisance modeling choices and on how the weighting and regression components interact.

\begin{algorithm}[htbp]
\caption{Augmented SCM}
\label{alg:asc}
\begin{algorithmic}[1]
\Require Control data $\{(X_j, Y_{j\tilde{t}})\}_{j=2}^N$, treated covariate $X_1$, importance matrix $V$, shrinkage parameter $\lambda$
\State \textbf{/* Stage 1: Nuisance function estimation */}
\State $\hat{m}_{\tilde{t}}(x) \gets \mathbb{E}[Y_{\tilde{t}}(0) \mid X = x]$
\State $\hat{w} \gets \arg\min_{w \in \Delta^{N-1}} \left\| X_1 - \sum_{j=2}^N w_j X_j \right\|_{V}^{2} + \lambda \mathcal{R}(w)$
\State $\tilde{Y}_{j\tilde{t}} \gets Y_{j\tilde{t}} - \hat{m}_{\tilde{t}}(X_j)$ \quad for controls $j = 2, \ldots, N$
\State \textbf{/* Stage 2: Augmented synthetic control */}
\State $\hat{\psi}^{\textsc{asc}}_{\tilde{t}} \gets \hat{m}_{\tilde{t}}(X_1) + \sum_{j=2}^{N} \hat{w}_j \bigl( Y_{j\tilde{t}} - \hat{m}_{\tilde{t}}(X_j) \bigr)$
\State \Return counterfactual $\hat{\psi}^{\textsc{asc}}_{\tilde{t}}$
\end{algorithmic}
\end{algorithm}

\newpage
\section{Proofs of Theorem~\ref{thm:boundedness_binary}}
\label{app:proof_bound}
Since $\hat w^\star \in \Delta^{N-1}$, we have $\hat w^\star_j \ge 0$ and $\sum_{j=2}^N \hat w^\star_j = 1$.
Thus, $\hat{\psi}^{\textsc{tsc}}_{{\tilde{t}}}$ is a convex combination of the observed control outcomes at time $\tilde{t}$.
When $Y_{j\tilde{t}}\in[a, b]$ for all $j=2,\ldots,N$, this convex-combination form immediately implies boundedness:
\begin{equation}
\hat \psi^{\textsc{tsc}}_{\tilde{t}}
= \sum_{j=2}^{N}\hat w^\star_j m_{\tilde{t}}(X_j)
=
\sum_{j=2}^{N}\hat w^\star_j Y_{j{\tilde{t}}}
\ge
\sum_{j=2}^{N}\hat w^\star_j\cdot a
=
a,
\label{eq:boundedness_lower}
\end{equation}
where the inequality uses $\hat w^\star_j \ge 0$ and $Y_{j\tilde{t}}\ge a$. Then, we yield
\begin{equation}
\hat{\psi}^{\textsc{tsc}}_{{\tilde{t}}}
= \sum_{j=2}^{N}\hat w^\star_j m_{\tilde{t}}(X_j)
=
\sum_{j=2}^{N}\hat w^\star_j Y_{j\tilde{t}}
\le
\sum_{j=2}^{N}\hat w^\star_j\cdot b
=
b,
\label{eq:boundedness_upper}
\end{equation}
where the inequality uses $Y_{j\tilde{t}}\le 1$ and $\sum_{j=2}^N \hat w^\star_j = 1$.

Therefore, $\hat \psi^{\textsc{tsc}}_{\tilde{t}}\in[a,b]$.
More generally, if $Y_{j\tilde{t}}\in[a,b]$, then $\hat{\psi}^{\textsc{tsc}}_{{\tilde{t}}}\in[a,b]$ by the same argument.

For the one-step estimator, note that it takes the affine form
\begin{align}
\hat\psi^{\text{1-step}}_{\tilde{t}}
=
\hat m_{\tilde{t}}(X_1)
+
\sum_{j=2}^{N}\hat w_j Y_{j{\tilde{t}}}
-
\sum_{j=2}^{N}\hat w_j \hat m_{\tilde{t}}(X_j),
\label{eq:onestep_affine}
\end{align}
which is a difference of two convex combinations plus $\hat m_{\tilde{t}}(X_1)$. Even if each term individually lies in $[a, b]$,
their signed combination in Eq.~(\ref{eq:onestep_affine}) need not lie in $[a,b]$.
\hfill$\square$

\newpage
\section{Population-level asymptotic normality}
\label{app:population_asymptotics}

Establishing asymptotic guarantees for TSC is challenging in the standard single-treated-unit setting, as documented in the synthetic control literature \citep{abadie.2021a}. In particular, the target counterfactual is a conditional mean evaluated at a fixed treated unit rather than a population-level functional. Increasing the number of control units does not provide additional independent observations of the treated unit~\citep{arkhangelsky.2024}, and standard semiparametric asymptotic arguments therefore do not directly apply. In this section, we consider a population-level extension under a more restrictive joint large-$N_0$, large-$T_0$ regime, where $N_0$ denotes the number of control units and $T_0$ denotes the number of pre-treatment periods. Importantly, the following theoretical results focus on this population-level extension under the stated assumptions, rather than the original single-treated-unit setting.

\textbf{Population-level setting:} We consider a longitudinal setting with a single treated unit and $N_0$ control units observed over $T_0$ pre-treatment periods. For a fixed prediction horizon $h \geq 1$, the corresponding post-treatment time is $t=T_0+h$. Let $X_{i,s}=Y_{i,s}\in\mathcal{X}\subseteq\mathbb{R}$ denote the state of unit $i$ at time $s$, and let $Y_{i,s+h}(0)$ denote its untreated potential outcome $h$ periods later. We observe two sources of data: the treated unit's pre-treatment states $\{X_{1,s}\}_{s=1}^{T_0}$ and the control-unit state-outcome pairs $\{(X_{j,T_0},Y_{j,T_0+h}(0))\}_{j=2}^{N_0+1}$. Let $P_1$ denote the marginal distribution of the treated unit's pre-treatment state process, and let $Q_0$ denote the joint distribution of $(X, Y_h)$ in the control population, where $X$ represents the state at the prediction origin and $Y_h$ the untreated outcome $h$ periods later. Let $P_0$ denote the marginal distribution of $X$ under $Q_0$. We define the horizon-specific outcome regression as $m_h(x)=\mathbb{E}_{Q_0}[Y_h\mid X=x]$. Throughout, we consider a joint asymptotic regime where $N_0, T_0\to\infty$ and impose the following additional assumptions.

\begin{assumption}[Population-level assumptions]
\label{ass:population}
For a fixed prediction horizon $h \geq 1$:
\begin{enumerate}
    \item[(i)] \emph{Stability}: The process $\{X_{1,s}\}_{s\geq 1}$ is strictly stationary with $X_{1,s}\sim P_1$ and satisfies suitable weak-dependence and moment conditions ensuring a central limit theorem. Moreover, $(X_{j,T_0},Y_{j,T_0+h}(0))\sim Q_0$ for all $T_0$, where $Q_0$ is invariant across prediction origins.
    \item[(ii)] \emph{Transportability}: For all $s$ and $P_1$-almost every $x\in\mathcal{X}$, $\mathbb{E}[Y_{1,s+h}(0)\mid X_{1,s}=x]=m_h(x)=\mathbb{E}_{Q_0}[Y_h\mid X=x]$.
    \item[(iii)] \emph{Overlap}: $P_1\ll P_0$, with $\omega(x)=\frac{\mathrm{d}\! P_1}{\mathrm{d}\! P_0}(x)\leq C<\infty$ for $P_0$-almost every $x\in\mathcal{X}$.
\end{enumerate}
\end{assumption}

Assumption~\ref{ass:population} characterizes the population-level setting: (i) stability ensures that the treated unit's pre-treatment observations follow a stationary, weakly dependent process with marginal distribution $P_1$, while the control-unit state-outcome distribution remains invariant across prediction origins; (ii) transportability ensures that the horizon-specific conditional untreated outcome regression is shared between the treated and control populations; and (iii) overlap ensures that the treated-state distribution is supported by the control-state distribution, with a bounded density ratio $\omega(x)=\mathrm{d}\! P_1(x)/\mathrm{d}\! P_0(x)$.

\textbf{Population-level estimand:} Under Assumption~\ref{ass:population}, we consider the population-level target
\begin{align}
\psi_h = \mathbb{E}_{P_1}[m_h(X)] = \mathbb{E}_{Q_0}[\omega(X)Y_h].
\end{align}
Unlike the counterfactual prediction for a fixed treated unit, this estimand averages over the treated-state distribution. It admits the following orthogonal estimating functional:
\begin{align}
\Psi_h(\bar{m},\bar{\omega}) = \mathbb{E}_{P_1}[\bar{m}(X)] + \mathbb{E}_{Q_0}[\bar{\omega}(X)\{Y_h-\bar{m}(X)\}].
\end{align}
This functional has a second-order bias structure in the outcome-regression and density-ratio estimation errors, providing a basis for population-level orthogonal estimation.

\textbf{Asymptotic normality:} Let $\{X_{1,s}\}_{s=1}^{T_0}$ denote the treated unit's pre-treatment states and $\{(X_{j,T_0},Y_{j,T_0+h}(0))\}_{j=2}^{N_0+1}$ denote the control-unit observations. We consider the following cross-fitted population-level orthogonal estimator:
\begin{align}
\hat{\psi}_h = \frac{1}{T_0}\sum_{s=1}^{T_0}\hat{m}_h(X_{1,s}) + \frac{1}{N_0}\sum_{j=2}^{N_0+1}\hat{\omega}(X_{j,T_0})\{Y_{j,T_0+h}(0)-\hat{m}_h(X_{j,T_0})\},
\end{align}
where the nuisance estimators are evaluated out of sample using aligned cross-fitting across the treated and control samples, with dependence-aware splitting for the treated-state process. For simplicity, fold-specific superscripts are suppressed.

\begin{theorem}[Asymptotic normality of the population-level orthogonal estimator]
\label{thm:population_normality}
Suppose Assumption~\ref{ass:population} holds. Assume further that (i) the control unit observations are independently and identically distributed according to $Q_0$ and are independent of the treated state process; (ii) the treated state contribution satisfies a central limit theorem, and the control-unit contribution has finite variance, with both asymptotic variances strictly positive; (iii) the cross-fitted nuisance estimators are consistent in $L^2(P_0)$, with centered nuisance-score empirical-process remainders of order $o_p(N_0^{-1/2}+T_0^{-1/2})$; and (iv) the nuisance estimators satisfy the product-rate condition
\begin{align}
\|\hat{m}_h-m_h\|_{L^2(P_0)}\|\hat{\omega}-\omega\|_{L^2(P_0)} = o_p\left(N_0^{-1/2}+T_0^{-1/2}\right),
\end{align}
where conditions (iii) and (iv) are understood to hold uniformly over the fixed number of cross-fitting folds. Then, as $N_0,T_0\rightarrow\infty$,
\begin{align}
\frac{\hat{\psi}_h-\psi_h}{\sqrt{\sigma_{1,h}^2/T_0+\sigma_{0,h}^2/N_0}} \xrightarrow{d} \mathcal{N}(0,1),
\end{align}
where
\begin{align}
\sigma_{1,h}^2 &= \lim_{T_0\rightarrow\infty}\frac{1}{T_0}\mathrm{Var}\left(\sum_{s=1}^{T_0}m_h(X_{1,s})\right),\\
\sigma_{0,h}^2 &= \mathrm{Var}_{Q_0}\left(\omega(X)\{Y_h-m_h(X)\}\right).
\end{align}
Here, $\sigma_{1,h}^2$ denotes the long-run variance of the treated-state contribution, while $\sigma_{0,h}^2$ denotes the variance of the control-unit contribution. Consequently, $\hat{\psi}_h-\psi_h=O_p(N_0^{-1/2}+T_0^{-1/2})$, corresponding to an effective convergence rate of $r_{N_0,T_0}=(N_0^{-1}+T_0^{-1})^{-1/2}\asymp\min\{\sqrt{N_0},\sqrt{T_0}\}$.
\end{theorem}

\textbf{Proof sketch:} By the density-ratio identity and the conditional mean-zero property of the control residuals, the population-level estimating functional satisfies
\begin{align}
\Psi_h(\bar{m},\bar{\omega})-\psi_h
= -\mathbb{E}_{P_0}\left[(\bar{\omega}(X)-\omega(X))(\bar{m}(X)-m_h(X))\right].
\end{align}
By the Cauchy--Schwarz inequality, the resulting bias is bounded by the product of the two nuisance estimation errors. Let $\mathbb{P}_{1,T_0}$ and $\mathbb{P}_{0,N_0}$ denote the empirical measures of the treated states and control observations, respectively. The estimation error admits the decomposition
\begin{align}
\hat{\psi}_h-\psi_h
={}&(\mathbb{P}_{1,T_0}-P_1)m_h
+\mathbb{P}_{0,N_0}\left[\omega(X)\{Y_h-m_h(X)\}\right]
+R_{N_0,T_0},
\end{align}
where the remainder consists of the second-order nuisance-product term and the centered nuisance-score empirical-process terms, interpreted foldwise. Under conditions (iii) and (iv), $R_{N_0,T_0}=o_p(N_0^{-1/2}+T_0^{-1/2})$. The first leading term satisfies a time-series central limit theorem with long-run variance $\sigma_{1,h}^2$, while the second satisfies an independent-sample central limit theorem with variance $\sigma_{0,h}^2$. Their independence and the negligible remainder yield the stated asymptotic normality.

\newpage
\section{Implementation details}

\subsection{Data generation}
\label{app:data_generation}

\textbf{General setup.}
We generate panels with one treated unit and $30$ control units. Each unit has $p=12$ covariates, and outcomes are observed for $T=50$ periods. The first $40$ periods are pre-treatment, and treatment starts at period $41$. We report results for prediction horizons $h\in\{1,5,10\}$. The treated unit is fixed before outcomes are generated.

\medskip
\noindent\textbf{Continuous outcomes.}
For each unit $i$, we draw $X_i\sim\mathcal{N}(0,I_p)$ and define the normalized time index $u_t=t/40$. The untreated outcome is generated as
\begin{align}
&g_i^{(d)}(t)=0.8X_{i1}-0.6X_{i2}+0.35X_{i3}+\left(0.25X_{i4}-0.15X_{i5}\right)\left(u_t-\frac{1}{2}\right)+q_i^{(d)}(t),\\
&Y_{it}(0)=0.2u_t+0.1\sin(2\pi u_t)+0.2g_i^{(d)}(t)+\varepsilon_{it},\qquad \varepsilon_{it}\sim\mathcal{N}(0,0.8^2).
\end{align}
The four DGPs differ only through the nonlinear component $q_i^{(d)}(t)$:
\begin{align}
&q_i^{(\mathrm{linear})}(t)=0,\\
&q_i^{(\mathrm{hinge})}(t)=\left\{0.8(X_{i1}-0.25)_+-0.6(-X_{i2}-0.25)_+\right\}(1+0.2u_t),\\
&q_i^{(\mathrm{quadratic})}(t)=\left\{\frac{0.5(X_{i1}^2-1)}{\sqrt{2}}+0.4X_{i1}X_{i2}\right\}(1+0.2u_t),\\
&q_i^{(\mathrm{time\text{-}varying})}(t)=0.7\sin(X_{i1})\sin(2\pi u_t)+0.5X_{i2}X_{i3}\cos(\pi u_t),
\end{align}
where $(a)_+=\max\{a,0\}$. For the treated unit, the observed post-treatment outcome is obtained by adding a constant effect of one to $Y_{it}(0)$.

\medskip
\noindent\textbf{Binary outcomes.}
For binary outcomes, covariates are generated around a common cohort center:
\begin{align}
&C\sim\mathcal{N}(0,I_p),\\
&X_i=C+s_d\eta_i,\qquad \eta_i\sim\mathcal{N}(0,I_p),
\end{align}
where $s_d=0.7$ for the linear DGP and $s_d=0.35$ for the other DGPs. We use
\begin{align}
&g_i^{(d)}(t)=0.8X_{i1}-0.6X_{i2}+0.35X_{i3}+0.9\left(0.5X_{i4}-0.3X_{i5}\right)\left(u_t-\frac{1}{2}\right)+\widetilde q_i^{(d)}(t),\\
&\widetilde q_i^{(\mathrm{linear})}(t)=0,\\
&\widetilde q_i^{(\mathrm{hinge})}(t)=1.2\left\{0.8(X_{i1}-0.25)_+-0.6(-X_{i2}-0.25)_+\right\}(1+0.3u_t),\\
&\widetilde q_i^{(\mathrm{quadratic})}(t)=1.2\left\{\frac{0.5(X_{i1}^2-1)}{\sqrt{2}}+0.4X_{i1}X_{i2}\right\}(1+0.3u_t),\\
&\widetilde q_i^{(\mathrm{time\text{-}varying})}(t)=1.2\left\{0.7\sin(X_{i1})\sin(2\pi u_t)+0.5X_{i2}X_{i3}\cos(\pi u_t)\right\}.
\end{align}
To add low-rank time variation, let
\begin{align}
&F_t=\begin{bmatrix}1 & u_t-\frac{1}{2} & \sin(2\pi u_t) & \cos(3\pi u_t)\end{bmatrix}^{\top},\\
&L_i=0.35X_{i,5:8}+\kappa_d Z_i,\qquad Z_i\sim\mathcal{N}(0,I_4),
\end{align}
where $\kappa_d=0.5$ for the linear DGP and $\kappa_d=0.2$ otherwise. The untreated probability is
\begin{align}
&M_{it}=0.3u_t+0.2\sin(2\pi u_t)+g_i^{(d)}(t)+\frac{1}{2}L_i^{\top}F_t,\\
&E_{i1}=\nu_{i1},\qquad E_{it}=0.2E_{i,t-1}+\sqrt{1-0.2^2}\,\nu_{it},\qquad \nu_{it}\sim\mathcal{N}(0,1),\\
&p_{it}(0)=\operatorname{expit}\left(-0.25+\frac{M_{it}+0.8E_{it}}{2}\right),\\
&Y_{it}(0)=\mathbf{1}\{U_{it}<p_{it}(0)\},\qquad U_{it}\sim\operatorname{Uniform}(0,1).
\end{align}
For the treated unit after period $40$, the treatment shifts the log odds by $1+0.5(t-40)/10$. The same uniform draw is used to construct the observed and untreated potential outcomes.

\subsection{Estimation and computation}
\label{app:implementation_details}

The initial SCM weights are obtained by minimizing the pre-treatment squared error over the probability simplex, using all available pre-treatment outcomes. When several solutions have essentially the same loss, we select the solution closest to uniform weights.

At each target period, the outcome regression uses the $12$ covariates and information from the last $10$ observed outcome periods. For continuous outcomes, the features include the full ten-period history, its mean, standard deviation, most recent value, three-period mean, and linear slope. Binary outcomes use the five summary features, except for the quadratic DGP, which uses the full history. We fit ridge regression with penalty $5$ and use five-fold cross-fitting across control units. For the binary linear DGP, the ridge penalty is $20$. Predictions for binary outcomes are not clipped.

Let $\widehat m_{jt}$ be the out-of-fold prediction for control unit $j$ at the target period and let $R_{jt}=Y_{jt}-\widehat m_{jt}$. Starting from the SCM weights $\widehat w_j^{\,0}$, the targeted weights are
\begin{align}
&\widehat w_j(\epsilon,\alpha) = \frac{\widehat w_j^{\,0}\exp\{\alpha\epsilon R_{jt}\}}{\sum_{k=1}^{N-1}\widehat w_k^{\,0}\exp\{\alpha\epsilon R_{kt}\}},\\
&\sum_{j=1}^{N-1}\widehat w_j(\epsilon,1)R_{jt} = 0,\\
&\widehat Y_{1t}^{\mathrm{TSC}}(0) = \sum_{j=1}^{N-1}\widehat w_j(\widehat\epsilon,\widehat\alpha)Y_{jt}.
\end{align}
The root is searched over $\epsilon\in[-20,20]$; if no root exists, we set $\widehat\epsilon=0$. The shrinkage parameter is selected from $\{0,0.1,0.25,0.5,0.75,1\}$ using squared prediction error over six pre-treatment pseudo-post periods. For the binary quadratic DGP, we fix $\alpha=1$.

We compare TSC with classical SCM, a plug-in estimator, and augmented SCM. All methods use the same initial SCM fit and outcome-regression specification. Performance is measured by squared error relative to the realized untreated outcome of the treated unit.

The experiments are implemented in Python 3.11 using NumPy 1.24.3, pandas 2.0.3, SciPy 1.11.1, and scikit-learn 1.3.0. All experiments are run serially with one numerical thread.

\newpage
\section{Additional experimental results}
\label{app:add_exp_results}
This section presents three sets of additional experiments to further evaluate TSC. We first assess its counterfactual prediction accuracy using semi-synthetic experiments based on real-world datasets. We then conduct an oracle ablation to investigate the sensitivity to outcome-regression estimation errors. Finally, we compare TSC with additional baselines and examine its flexibility in incorporating alternative outcome-regression backbones.

\subsection{Semi-synthetic experiments}

To provide a direct quantitative comparison using real trajectories, we added additional semi-synthetic experiments using both datasets. We retain the observed control units, covariates, and outcome trajectories, and designate untreated units as pseudo-treated. Their observed post-intervention trajectories therefore provide ground truth for evaluating counterfactual prediction accuracy.

On the tobacco panel, TSC achieves an RMSE of $11.728 \pm 3.617$ at lag 1, $23.616 \pm 8.226$ at lag 5, and $16.388 \pm 5.174$ at lag 10. At lag 10, the corresponding RMSE values for classical SCM, augmented SCM, and the plug-in estimator are $39.127 \pm 8.652$, $27.304 \pm 7.359$, and $19.240 \pm 5.485$, respectively. On the binary turnout panel, TSC achieves $0.0390 \pm 0.0110$, $0.0971 \pm 0.0248$, and $0.0748 \pm 0.0159$ across the three lag settings. TSC achieves the lowest RMSE in all six settings.

We note that the differences become larger at longer horizons. This pattern is consistent with the targeting step being particularly useful when inaccuracies in the initial pre-treatment weights have a larger effect on longer-horizon predictions.

These experiments also clarify the role of boundedness. The empirical contribution is not only that TSC remains within the control-outcome range, but that this structural guarantee is obtained without sacrificing predictive accuracy. TSC matches or improves upon the unconstrained augmented SCM.

\begin{table*}[htbp]
\centering

\begin{minipage}[t]{0.49\textwidth}
\centering
\resizebox{\linewidth}{!}{%
\begin{tabular}{lccc}
\toprule
Method & Lag 1 (1988) & Lag 5 (1992) & Lag 10 (1997) \\
\midrule
Classical SCM & $14.000 \pm 6.092$ & $26.699 \pm 8.883$ & $39.127 \pm 8.652$ \\
Plugin-in estimator & $11.872 \pm 3.581$ & $24.585 \pm 9.054$ & $19.240 \pm 5.485$ \\
Augmented SCM & $11.981 \pm 3.746$ & $23.990 \pm 10.678$ & $27.304 \pm 7.359$ \\
TSC (\emph{ours}) & $\mathbf{11.728 \pm 3.617}$ & $\mathbf{23.616 \pm 8.226}$ & $\mathbf{16.388 \pm 5.174}$ \\
\bottomrule
\end{tabular}%
}
\caption{Tobacco: Continuous outcome}
\label{tab:tobacco}
\end{minipage}
\hfill
\begin{minipage}[t]{0.49\textwidth}
\centering
\resizebox{\linewidth}{!}{%
\begin{tabular}{lccc}
\toprule
Method & Lag 1 (1992) & Lag 5 (1996) & Lag 10 (2001) \\
\midrule
Classical SCM & $0.0495 \pm 0.0135$ & $0.0993 \pm 0.0249$ & $0.0904 \pm 0.0184$ \\
Plugin-in estimator & $0.0404 \pm 0.0145$ & $0.1982 \pm 0.0306$ & $0.0968 \pm 0.0237$ \\
Augmented SCM & $0.0402 \pm 0.0290$ & $0.1077 \pm 0.0584$ & $0.0816 \pm 0.0502$ \\
TSC (\emph{ours}) & $\mathbf{0.0390 \pm 0.0110}$ & $\mathbf{0.0971 \pm 0.0248}$ & $\mathbf{0.0748 \pm 0.0159}$ \\
\bottomrule
\end{tabular}%
}
\caption{Turnout: Binary outcome}
\label{tab:turnout}
\end{minipage}

\end{table*}

\subsection{Oracle outcome-regression ablation}
\label{subsec:oracle_ablation}

To further investigate whether the advantage of TSC arises solely from its convex-hull constraint, we conduct an oracle ablation under the latent-factor data-generating process (DGP). Specifically, we replace the estimated conditional mean $\hat{m}_t$ with the true conditional mean $m_t^\star$ for both Augmented SCM and TSC. We perform 50 repetitions for each prediction horizon $h \in \{1,5,10\}$ and evaluate the oracle gain, defined as $G_h = \mathrm{RMSE}_h(\hat{m}_t) - \mathrm{RMSE}_h(m_t^\star)$. A larger oracle gain indicates a greater improvement in predictive accuracy when the true outcome model is available.

\begin{table}[htbp]
\centering
\begin{tabular}{lccc}
\toprule
Method & $h=1$ & $h=5$ & $h=10$ \\
\midrule
Augmented SCM & 1.24 & 1.18 & 0.61 \\
TSC (\emph{ours}) & 0.53 & 0.54 & 0.49 \\
\bottomrule
\end{tabular}
\vspace{0.2cm}
\caption{Oracle gain in RMSE under the latent-factor DGP across different prediction horizons.}
\label{tab:oracle_ablation}
\end{table}

As shown in Table~\ref{tab:oracle_ablation}, the two methods respond differently to outcome-regression quality. Augmented SCM benefits substantially from the oracle outcome model, with gains of up to $1.24$, whereas TSC exhibits smaller improvements, with gains of at most $0.54$. Once both methods are provided with the true conditional mean, their predictive accuracy becomes comparable. This suggests that the performance gap under estimated outcome models is driven, at least in part, by differences in how outcome-regression estimation errors propagate into counterfactual predictions. In particular, Augmented SCM directly relies on the treated-unit prediction $\hat{m}_t(X_1)$, which can be sensitive to extrapolation errors, whereas TSC avoids evaluating the outcome model at the treated unit.

These findings suggest that the advantage of TSC is not solely attributable to its convex-hull constraint. Rather, its construction also reduces sensitivity to outcome-regression estimation errors in finite samples. We emphasize that the oracle experiment serves as a diagnostic analysis rather than an attainable practical setting, since the true conditional mean is generally unknown. The results also indicate that when the outcome model can be estimated accurately, Augmented SCM remains competitive with TSC.

\subsection{Additional Baselines and Backbone Ablation}
\label{subsec:additional_baselines}

To further evaluate the predictive performance of TSC, we extend our comparison to three additional baselines: Generalized Synthetic Control (GSC), Matrix Completion (MC), and SyncTwin. We conduct experiments under the synthetic quadratic data-generating process (DGP) and report the RMSE across three prediction horizons, $h \in \{1,5,10\}$. As shown in Table~\ref{tab:additional_baselines}, TSC consistently achieves the lowest RMSE among the four methods across all three horizons.

\begin{table}[htbp]
\centering
\begin{tabular}{lccc}
\toprule
Method & Lag 1 & Lag 5 & Lag 10 \\
\midrule
GSC & $10.612 \pm 0.405$ & $10.680 \pm 0.468$ & $10.670 \pm 0.478$ \\
MC & $10.668 \pm 0.401$ & $10.667 \pm 0.418$ & $10.638 \pm 0.413$ \\
SyncTwin & $10.921 \pm 1.341$ & $9.825 \pm 1.342$ & $9.925 \pm 1.343$ \\
TSC (\emph{ours}) & $\mathbf{10.150 \pm 1.286}$ & $\mathbf{9.733 \pm 1.127}$ & $\mathbf{9.923 \pm 1.328}$ \\
\bottomrule
\end{tabular}
\vspace{0.2cm}
\caption{Comparison with additional baselines under the synthetic quadratic DGP.}
\label{tab:additional_baselines}
\end{table}

Furthermore, TSC is model-agnostic and can incorporate different outcome models within its targeting framework. To examine this flexibility, we replace the default MLP-based outcome estimator $\hat{m}_t(\cdot)$ with GSC, MC, and SyncTwin, while keeping the targeting step unchanged. Table~\ref{tab:backbone_ablation} presents the corresponding results under the same quadratic DGP.

\begin{table}[htbp]
\centering
\begin{tabular}{lccc}
\toprule
Method & Lag 1 & Lag 5 & Lag 10 \\
\midrule
TSC + GSC $m_t(X)$ & $9.991 \pm 1.656$ & $10.015 \pm 1.662$ & $10.007 \pm 1.659$ \\
TSC + MC $m_t(X)$ & $\mathbf{9.267 \pm 1.901}$ & $\mathbf{9.270 \pm 1.902}$ & $\mathbf{9.259 \pm 1.898}$ \\
TSC + SyncTwin $m_t(X)$ & $9.312 \pm 0.609$ & $9.599 \pm 0.614$ & $10.184 \pm 0.596$ \\
\bottomrule
\end{tabular}
\vspace{0.2cm}
\caption{Backbone ablation of TSC under the synthetic quadratic DGP.}
\label{tab:backbone_ablation}
\end{table}

Among the alternative backbones, MC consistently yields the lowest RMSE across all three horizons. In particular, at lag 10, incorporating MC reduces the RMSE from $9.923 \pm 1.328$ with the default MLP backbone to $9.259 \pm 1.898$. These results demonstrate that TSC is not restricted to a specific outcome-regression architecture and can benefit from alternative prognostic estimators while retaining the same targeting procedure.

\newpage
\section{Discussion}
\label{app:discussion}

\textbf{Conclusion:} We propose a targeted synthetic control (TSC) method that refines initial synthetic-control weights through a TMLE-style targeting procedure while preserving the convex-combination structure of the counterfactual estimator. By incorporating outcome-regression information through residual balancing, TSC maintains interpretable and bounded counterfactual predictions. Our experiments demonstrate that TSC can improve counterfactual prediction accuracy across a range of synthetic and real-world settings.

\textbf{Limitations:} The advantages of TSC are not universal. In particular, TSC can be especially beneficial when outcome models are difficult to estimate accurately, as its final predictions do not directly rely on outcome-model predictions for the treated unit. However, when the outcome model is accurately estimated, augmented SCM can achieve comparable performance. Moreover, our theoretical guarantees primarily concern finite-sample structural properties, including residual balancing and convex-hull preservation. Establishing general asymptotic properties remains challenging in the single-treated-unit setting, where standard large-sample arguments do not directly apply. Asymptotic normality may be established in special regimes with an increasing number of control units and pre-treatment periods, but this would require additional assumptions on temporal stability, transportability, and overlap.

\textbf{Future directions:} An important direction for future research is to establish asymptotic guarantees and develop valid statistical inference for TSC under suitable large-sample regimes. Beyond theoretical developments, the proposed targeting framework could also be extended to other settings, such as single-arm clinical trials in medicine using external or historical controls, where interpretable and bounded counterfactual estimates may be particularly valuable.

\end{document}